\documentclass[11pt]{article}

\usepackage{graphicx} 

\usepackage[square]{natbib}

\usepackage{wrapfig}
\usepackage{caption}
\usepackage{subcaption}
\usepackage{booktabs}
\usepackage{amsmath}

\usepackage{hyperref}

\usepackage{fullpage}

\usepackage{algorithm}
\usepackage{algorithmic}

\usepackage[usenames,dvipsnames]{xcolor}

\newcommand{\matt}[1]{\textcolor{red}{Matt: #1}}
\newcommand{\bryan}[1]{\textcolor{green}{Bryan: #1}}
\newcommand{\sj}[1]{\textcolor{blue}{SJ: #1}}

 \renewcommand{\matt}[1]{}
 \renewcommand{\bryan}[1]{}
 \renewcommand{\sj}[1]{}

\newcommand{\Ds}{\mathcal D}
\newcommand{\Es}{\mathcal E}
\newcommand{\Is}{\mathcal I}
\newcommand{\Ls}{\mathcal L}

\newcommand{\Ps}{\mathcal P}

\newcommand{\Xs}{\mathcal X}
\newcommand{\Ys}{\mathcal Y}

\newcommand{\reals}{\mathbb R}

\usepackage{amsmath,amsfonts}
\DeclareMathOperator{\E}{\mathbb{E}}

\DeclareMathOperator{\Var}{\mathrm{Var}}

\newcommand{\argmin}{\operatornamewithlimits{argmin}}
\newcommand{\argmax}{\operatornamewithlimits{argmax}}

\usepackage{amsthm}
\newtheorem{theorem}{Theorem}[section]
\newtheorem{lemma}{Lemma}[section]
\newtheorem{corollary}{Corollary}[section]

\theoremstyle{definition}
\newtheorem{definition}{Definition}[section]

\newtheorem{remark}{Remark}[section]
\newtheorem{fact}{Fact}[section]

\hypersetup{
    colorlinks,
    linkcolor={red!50!black},
    citecolor={blue!50!black},
    urlcolor={blue!80!black}
}

\title{Distributionally Robust Submodular Maximization}

\author{Matthew Staib\thanks{Equal contribution} \\ MIT CSAIL \\ mstaib@mit.edu \and 
Bryan Wilder\footnotemark[1] \\ USC \\ bwilder@usc.edu \and
Stefanie Jegelka \\ MIT CSAIL \\ stefje@csail.mit.edu}

\date{}

\begin{document}

\maketitle

\begin{abstract} 
Submodular functions have applications throughout machine learning, but in many settings, we do not have \emph{direct} access to the underlying function $f$. We focus on \emph{stochastic} functions that are given as an expectation of functions over a distribution $P$. In practice, we often have only a limited set of samples $f_i$ from $P$. The standard approach \emph{indirectly} optimizes $f$ by maximizing the sum of $f_i$. However, this ignores generalization to the true (unknown) distribution. In this paper, we achieve better performance on the actual underlying function $f$ by directly optimizing a combination of bias and variance. Algorithmically, we accomplish this by showing how to carry out distributionally robust optimization (DRO) for submodular functions, providing efficient algorithms backed by theoretical guarantees which leverage several novel contributions to the general theory of DRO. We also show compelling empirical evidence that DRO improves generalization to the unknown stochastic submodular function.
\end{abstract} 


\section{Introduction}
Submodular functions have natural applications in many facets of machine learning and related areas, e.g. dictionary learning~\citep{ICML2011Das_542}, influence maximization~\citep{kempe_maximizing_2003,domingos2001mining}, data summarization~\citep{lin2011class}, probabilistic modeling~\citep{djolonga2014map} and diversity~\citep{kulesza2012dpp}. In these settings, we have a set function $f(S)$ over subsets $S$ of some ground set of items $V$, and seek $S^*$ so that $f(S^*)$ is as large or small as possible. While optimization of set functions is hard in general, submodularity enables exact minimization and approximate maximization in polynomial time. 


In many settings, the submodular function we wish to optimize has additional structure, which may present both challenges and an opportunity to do better. In particular, the stochastic case has recently gained attention, where we wish to optimize $f_P(S):=\E_{f \sim P}[f(S)]$ for some distribution $P$. The most naive approach is to draw many samples from $P$ and optimize their average; this is guaranteed to work when the number of samples is very large. Much recent work has focused on more computationally efficient gradient-based algorithms for stochastic submodular optimization~\citep{Karimi2017stochastic,Mokhtari2017conditional,Hassani2017gradient}. All of this work assumes that we have access to a sampling oracle for $P$ that, on demand, generates as many iid samples as are required. But in many realistic settings, this assumption fails: we may only have access to historical data and not a simulator for the ground truth distribution. Or, computational limitations may prevent drawing many samples if $P$ is expensive to simulate.

Here, we address this gap and consider the maximization of a stochastic submodular function given access to a \emph{fixed} set of samples $f_1,\dots,f_n$ that form an empirical distribution $\hat P_n$. This setup introduces elements of statistical learning into the optimization. Specifically, we need to ensure that the solution we choose generalizes well to the unknown distribution $P$. A natural approach is to optimize the empirical estimate $\hat{f}_n = \frac{1}{n}\sum_{i = 1}^n f_i$, analogous to empirical risk minimization. The average $\hat{f}_n$ is an unbiased estimator of $f_P$, and when $n$ is very large, generalization is guaranteed by standard concentration bounds. We ask: is it possible to do better, particularly in the realistic case where $n$ is small (at least relative to the variance of $P$)? In this regime, a biased estimator could achieve much lower variance and thereby improve optimization.


Optimizing this bias-variance tradeoff is at the heart of statistical learning.
Concretely, instead of optimizing the finite sum, we will optimize the variance-regularized objective $f_{\hat P_n}(S) - C_1 \sqrt{\Var_{\hat P_n}(f(S))/n}$. When the variance is high, this term dominates a standard high-probability lower bound on $f_P(S)$. Unfortunately, direct optimization of this bound is in general intractable: even if all $f_i$ are submodular, their variance need not be~\citep{staib2017robust}.

In the continuous setting, it is known that variance regularization is equivalent to solving a distributionally robust problem, where an adversary perturbs the empirical sample within a small ball~\citep{gotoh2015robust,lam2016robust,Namkoong2017}. The resulting maximin problem is particularly nice in the concave case, since the pointwise minimum of concave functions is still concave and hence global optimization remains tractable. However, this property does not hold for submodular functions, prompting much recent work on robust submodular optimization~\citep{krause2011randomized,chen2017robust,staib2017robust,anari2017robust,wilder2018equilibrium,orlin16_robust,bogunovic2017robust}.

In this work, \textbf{1.} we show that, perhaps surprisingly, variance-regularized submodular maximization is both tractable and scalable.
\textbf{2.} We give a theoretically-backed algorithm for distributionally robust submodular optimization which substantially improves over a naive application of previous approaches for robust submodular problems. Along the way, \textbf{3.} we develop improved technical results for general (non-submodular) distributionally robust optimization problems, including both improved algorithmic tools and more refined structural characterizations of the problem. For instance, we give a more complete characterization of the relationship between distributional robustness and variance regularization. \textbf{4.} We verify empirically that in many real-world settings, variance regularization enables better generalization from fixed samples of a stochastic submodular function, particularly when the variance is high.

\matt{stefanie: maybe as a subroutine for one of the distributed submodular maximization papers?}\sj{This may not change the expectation but yield/improve high probability bounds. Check Alina \& Huy's paper, or the Mirrokni one. (Don't remember if they have high-prob.\ results in there, but if it works could be a cool addition.}
\matt{seems there are both in expectation only}


\paragraph{Related Work.}

We build on and significantly extend a recent line of research in statistical learning and optimization that develops a relationship between distributional robustness and variance-based regularization~\citep{maurer2009empirical,gotoh2015robust,lam2016robust,duchi2016statistics,Namkoong2017}. While previous work has uniformly focused on the continuous (and typically convex) case, here we address \emph{combinatorial} problems with submodular structure, requiring further technical developments. As a byproduct, we better characterize the behavior of the DRO problem under low sample variance (which was left open in previous work), show conditions under which the DRO problem becomes smooth, and provide improved algorithmic tools which apply to general DRO problems.  


Another related area is robust submodular optimization~\citep{krause2011randomized,chen2017robust,staib2017robust,anari2017robust,wilder2018equilibrium,orlin16_robust,bogunovic2017robust}. Much of this recent surge in interest is inspired by applications to robust influence maximization~\citep{chen_robust_2016,he_robust_2016,lowalekar_robust_2016-1}. Existing work aims to maximize the minimum of a set of submodular functions, but does not address the \emph{distributionally} robust optimization problem where an adversary perturbs the empirical distribution. We develop scalable algorithms, accompanied by approximation guarantees, for this case. Our algorithms improve both theoretically and empirically over naive application of previous robust submodular optimization algorithms to DRO. 
Further, our work is motivated by the connection between distributional robustness and generalization in learning, which has not previously been studied for submodular functions. \citet{stan2017probabilistic} study generalization in a related combinatorial problem, but they do not explicitly balance bias and variance, and the goal is different: they seek a smaller ground set which still contains a good subset for each user in the population. \matt{does placement of stan paper here make sense?}

A complementary line of work concerns \emph{stochastic} submodular optimization~\citep{Mokhtari2017conditional,Hassani2017gradient,Karimi2017stochastic}, where we have to a sampling oracle for the underlying function. We draw on stochastic optimization tools, but address problems where only a fixed dataset is available.

Our motivation also relates to optimization from samples. There, we have access to values of a fixed unknown function on inputs sampled from a distribution. The question is whether such samples suffice to (approximately) optimize the function. \citet{balkanski2015limitations,balkanski2016power} prove hardness results for general submodular maximization, with positive results for functions with bounded curvature. We address a different model where the underlying function itself is stochastic and we observe realizations of it. Hence, it is possible to well-approximate the optimization problem from polynomial samples and the challenge is to construct algorithms that make more effective use of data.   

\matt{\citet{yue2011linear,El-Arini2009turning} describe a kind of recommendation problem with submodular functions. \citet{stan2017probabilistic,balkanski2016learning} both consider the related problem of choosing a good ground set from which a good subset can be chosen for each sample. first two can be cited in facloc section, third is a statistical problem}


\section{Stochastic Submodular Functions and Distributional Robustness} \label{sec:math}
A set function $f : 2^V \to \reals$ is \emph{submodular} if it satisfies \emph{diminishing marginal gains}: for all $S\subseteq T$ and all $i \in V\setminus T$, it holds that $f(S\cup\{i\}) - f(S) \geq f(T\cup\{i\}) - f(T)$. It is \emph{monotone} if $S \subseteq T$ implies $f(S) \leq f(T)$.
Let $P$ be a distribution over monotone submodular functions $f$. We assume that each function is normalized and bounded, i.e., $f(\emptyset) = 0$ and $f(S) \in [0, B]$ almost surely for all subsets $S$. We seek a subset $S$ that maximizes
\begin{equation}
\label{eq:stochastic-submodular}
	f_P(S) := \E_{f \sim P}[f(S)]
\end{equation}
subject to some constraints, e.g., $\lvert S \rvert \leq k$. We call the function $f_P(S)$ a \emph{stochastic submodular function}. Such functions arise in many domains; we begin with two specific motivating examples.

\subsection{Stochastic Submodular Functions}

\paragraph{Influence Maximization.}
Consider a graph $G = (V, E)$ on which influence propagates. We seek to choose an initial seed set $S \subseteq V$ of influenced nodes to maximize the expected number subsequently reached.
Each edge can be either active, meaning that it can propagate influence, or inactive. A node is influenced if it is reachable from $S$ via active edges. Common diffusion models specify a distribution of active edges, e.g., the Independent Cascade Model (ICM), the Linear Threshold Model (LTM), and generalizations thereof. Regardless of the specific model, each can be described by the distribution of ``live-edge graphs'' induced by the active edges $\Es$~\citep{kempe_maximizing_2003}. Hence, the expected number of influenced nodes $f(S)$ can be written as an expectation over live-edge graphs:
$
	f_{\text{IM}}(S) = \E_{\Es} [ f(S; \Es) ].
$
The distribution over live-edge graphs induces a distribution $P$ over functions $f$ as in equation~\eqref{eq:stochastic-submodular}.

\paragraph{Facility Location.}
\matt{is there a standard cite for facloc?}\sj{you would want one for this stochastic version; not sure}
Fix a ground set $V$ of possibile facility locations $j$. Suppose we have a (possibly infinite as in~\citep{stan2017probabilistic}) number of demand points $i$ drawn from a distribution $\Ds$. The goal of \emph{facility location} is to choose a subset $S \subset V$ that covers the demand points as well as possible. Each demand point $i$ is equipped with a vector $r^i \in \reals^{\lvert V \rvert}$ describing how well point $i$ is covered by each facility $j$. We wish to maximize:
$
	f_{\text{facloc}}(S) = \E_{i \sim \Ds} \left[\max\nolimits_{j \in S} r^i_j \right].
$
Each $f(S) = \max_{j\in S} r_j$ is submodular, and $\Ds$ induces a distribution $P$ over the functions $f(S)$ as in equation~\eqref{eq:stochastic-submodular}.

\subsection{Optimization and Empirical Approximation}
Two main issues arise with stochastic submodular functions. First, simple techniques such as the greedy algorithm become impractical since we must accurately compute marginal gains. 
Recent alternative algorithms~\citep{Karimi2017stochastic,Mokhtari2017conditional,Hassani2017gradient} make use of additional, specific information about the function, such as efficient gradient oracles for the multilinear extension. A second issue has so far been neglected: the degree of access we have to the underlying function (and its gradients). In many settings, we only have access to a limited, fixed number of samples, either because these samples are given as observed data
or because sampling the true model is computationally prohibitive. 

Formally, instead of the full distribution $P$, we have access to an empirical distribution $\hat P_n$ composed of $n$ samples $f_1, \dots, f_n \sim P$. One approach is to optimize
\begin{equation}
	f_{\hat P_n} = \E_{f \sim \hat P_n}[f(S)] = \frac1n \sum\nolimits_{i=1}^n f_i(S),
\end{equation}
and hope that $f_{\hat P_n}$ adequately approximates $f_P$. This is guaranteed when $n$ is sufficiently large. E.g., in influence maximization, for $f_{\hat P_n}(S)$ to approximate $f_P(S)$ within error $\epsilon$ with probability $1-\delta$, \citet{kempe2015maximizing} show that $O\left(\frac{\lvert V \rvert^2}{\epsilon^2}\log\frac{1}{\delta}\right)$ samples suffice. To our knowledge, this is the tightest general bound available. Still, it easily amounts to thousands of samples even for small graphs; in many applications we would not have so much data. 

The problem of maximizing $f_P(S)$ from samples greatly resembles statistical learning. Namely, if the $f_i$ are drawn iid from $P$, then we can write
\begin{equation}
	f_P(S) \geq f_{\hat P_n}(S) - C_1 \sqrt{\frac{\Var_P{(f(S))}}{n}} - \frac{C_2}{n}
\end{equation}
for each $S$ with high probability, where $C_1$ and $C_2$ are constants that depend on the problem. For instance, if we want this bound to hold with probability $1-\delta$, then applying the Bernstein bound (see Appendix \ref{appendix:bias-variance}) yields $C_1 \leq \sqrt{2\log\frac{1}{\delta}}$ and $C_2 \leq \frac{2B}{3}\log\frac{1}{\delta}$ (recall that $B$ is an upper bound on $f(S)$). Given that we have only finite samples, it would then be logical to directly optimize
\begin{equation}
	\label{eq:var-regularized-objective}
	f_{\hat P_n}(S) - C_1 \sqrt{\Var_{\hat P_n}{(f(S))}/n},
\end{equation}
where $\Var_{\hat P_n}$ refers to the empirical variance over the sample. This would allow us to directly optimize the tradeoff between bias and variance. However, even when each $f$ is submodular, the variance-regularized objective need not be~\citep{staib2017robust}.

\subsection{Variance regularization via distributionally robust optimization}
While the optimization problem \eqref{eq:var-regularized-objective} is not directly solvable via submodular optimization, we will see next that distributionally robust optimization (DRO) can help provide a tractable reformulation. In DRO, we seek to optimize our function in the face of an adversary who perturbs the empirical distribution within an uncertainty set $\Ps$:
\begin{equation}
	\label{eq:general-dro-problem}
	\max_S \min_{\tilde P \in \Ps} \E_{f \sim \tilde P} [f(S)].
\end{equation}
We focus on the case when the adversary set $\Ps$ is a $\chi^2$ ball:
\begin{definition}
	The $\chi^2$ divergence between distributions $P$ and $Q$ is 
	\begin{equation}
		D_\phi(P || Q) = \frac12 \int \left( dP/dQ - 1 \right)^2 \, dQ.
	\end{equation} 
	The $\chi^2$ uncertainty set around an empirical distribution $\hat P_n$ is
	\begin{equation}
		\Ps_{\rho,n} = \{ \tilde P : D_\phi(\tilde P || \hat P_n) \leq \rho/n \}.
	\end{equation}
	When $\hat P_n$ corresponds to an empirical sample $Z_1,\dots,Z_n$, we encode $\tilde P$ by a vector $p$ in the simplex $\Delta_n$ and equivalently write
	\begin{equation}
		\Ps_{\rho,n} = \left\{ p \in \Delta_n : \tfrac12 \lVert np - \mathbf 1 \rVert_2^2 \leq \rho \right\}.
	\end{equation}
\end{definition}
In particular, maximizing the variance-regularized objective~\eqref{eq:var-regularized-objective} is equivalent to solving a distributionally robust problem when the sample variance is high enough: \sj{quickly check if this already occurs in the Gotoh or Lam papers} \matt{they are both asymptotic only}
\begin{theorem}[{modified from \citep{Namkoong2017}}]
\label{thm:dro-var-equiv}
Fix $\rho \geq 0$, and let $Z \in [0,B]$ be a random variable (i.e. $Z = f(S)$). 
Write $s_n^2 = \Var_{\hat P_n}(Z)$ and let $OPT = \inf_{\tilde P \in \Ps_{\rho,n}} \E_{\tilde P}[Z]$.
Then
\begin{equation}
	\left( \sqrt{\frac{2\rho}{n} s_n^2} - \frac{2B\rho}{n} \right)_+ \leq
	\E_{\hat P_n}[Z] - OPT \leq \sqrt{\frac{2\rho}{n} s_n^2}.
\end{equation}
Moreover, if $s_n^2 \geq 2\rho (\max_i z_i - \overline z_n)^2 / n$, then $OPT = \E_{\hat P_n}[Z] - \sqrt{2\rho s_n^2/n}$,
i.e., DRO is exactly equivalent to variance regularization.
\end{theorem}

In several settings, \citet{Namkoong2017} show this holds with high probability, by requiring high population variance $\Var_P(Z)$ and applying concentration results. Following a similar strategy, we obtain a corresponding result for submodular functions:

\begin{lemma}
	Fix $\delta$, $\rho$, $\lvert V \rvert$ and $k\geq 1$. Define the constant 
	\begin{equation*}
		M = \max\left\{\sqrt{32\rho/7}, \sqrt{36\left(\log\left(1/\delta\right) + |V| \log(25 k)\right)}\right\}.
	\end{equation*}
	For all $S$ with $\lvert S \rvert \leq k$ and
	$\Var_{\Ps}(f_P(S)) \geq \frac{B}{\sqrt{n}}M$, 
	DRO is exactly equivalent to variance regularization with combined probability at least $1 - \delta$. 
\end{lemma}

This result is obtained as a byproduct of a more general argument that applies to all points in a fractional relaxation of the submodular problem (see Appendix \ref{appendix:dro-variance-equivalence}) and shows equivalence of the two problems when the variance is sufficiently high. However, it is not clear what the DRO problem yields when the sample variance is too small. We give a more precise characterization of how the DRO problem behaves under arbitrary variance: 


\begin{lemma}\label{lem:arbitrary_var}
Let $\rho < n(n-1)/2$. Suppose all $z_1,\dots,z_n$ are distinct, with $z_1 < \dots < z_n$. Define $\alpha(m,n,\rho) = 2\rho m/n^2 + m/n - 1$, and let $\Is = \{ m \in \{1,\dots,n\} : \alpha(m,n,\rho) > 0 \}$. Then, $\inf_{\tilde P \in \Ps_{\rho,n}} \E_{\tilde P}[Z]$ is equal to
\begin{align*}
  \min_{m \in \Is} \left\{ \overline z_m - \min \left\{ \sqrt{\alpha(m,n,\rho) s_m^2}, \; \frac{s_m^2}{z_m - \overline z_m} \right\} \right\}
  \leq \E_{\hat P_n}[Z] - \min \left\{ \sqrt{\frac{2\rho s_n^2}{n}}, \; \frac{s_n^2}{z_n - \overline z_n} \right\},
\end{align*}
where $\hat P_m$ denotes the uniform distribution on $z_1, \dots, z_m$, $\overline z_m = \E_{\hat P_m}[Z]$, and $s_m^2 = \Var_{\hat P_m}(Z)$.
\end{lemma}
The inequality holds since $n$ is always in $\Is$ and $\alpha(n,n,\rho)=2\rho/n$.
As in Theorem~\ref{thm:dro-var-equiv},
when the variance $s_n^2 \geq 2\rho / n \cdot (z_n - \overline z_n)^2$, we recover the exact variance expansion.
We show Lemma~\ref{lem:arbitrary_var} by developing an exact algorithm for optimization over the $\chi^2$ ball (see Appendix~\ref{appendix:chi-squared-linear-oracle}).

Finally, we apply the equivalence of DRO and variance regularization to obtain a surrogate optimization problem. Fix $S$, and let $Z$ be the random variable induced by $f(S)$ with $f \sim P$. Theorem~\ref{thm:dro-var-equiv} in this setting suggests that instead of directly optimizing equation~\eqref{eq:var-regularized-objective}, we can instead solve
\begin{equation}
	\label{eq:robust-objective}
	\max_S \min_{\tilde P \in \Ps_{\rho,n}} \E_{f \sim \tilde P}[f(S)] = \max_S \min_{p \in \Ps_{\rho,n}} \sum_{i=1}^n p_i f_i(S).
\end{equation}


\section{Algorithmic Approach}
\label{sec:alg}
Even though each $f_i(\cdot)$ is submodular, it is not obvious how to solve Problem~\eqref{eq:robust-objective}:
robust submodular maximization is in general inapproximable, i.e. no polynomial-time algorithm can guarantee a positive fraction of the optimal value unless P = NP~\citep{krause_robust_2008}. Recent work has sought tractable relaxations~\citep{staib2017robust,krause_robust_2008,wilder2018equilibrium,anari2017robust,orlin16_robust,bogunovic2017robust}, but these either do not apply or yield much weaker results in our setting.
We consider a relaxation of robust submodular maximization 
that returns a near-optimal \emph{distribution}
over subsets $S$ (as in~\citep{chen2017robust,wilder2018equilibrium}). That is, we solve the robust problem $\max_\Ds \min_{i \in [m]} \E_{S \sim \Ds}[f_i(S)]$ where $\Ds$ is a distribution over sets $S$.  Our strategy, based on ``continuous greedy'' ideas, extends the set function $f$ to a continuous function $F$, then maximizes a robust problem involving $F$ via continuous optimization.

\paragraph{Multilinear extension.} One canonical extension of a submodular function $f$ to the continuous domain 
is the \emph{multilinear extension}. The multilinear extension $F : [0,1]^{\lvert V \rvert} \to \reals$ of $f$ is defined as $F(x) = \sum_{S \subseteq V} f(S) \prod_{i \in S} x_i \prod_{j \not\in S} (1 - x_j)$. That is, $F(x)$ is the expected value of $f(S)$ when each item $i$ in the ground set is included in $S$ independently with probability $x_i$. 
A crucial property of $F$ (that we later return to) is that it is a continuous \emph{DR-submodular} function:
\begin{definition}
A continuous function $F : \Xs \to \reals$ is DR-submodular if, for all $x \leq y \in \Xs$, $i\in[n]$, and $\gamma > 0$ so that $x + \gamma e_i$ and $y+\gamma e_i$ are still in $\Xs$, we have
$F(x + \gamma e_i) - F(x) \geq F(y + \gamma e_i) - F(y)$.
\end{definition}
Essentially, a DR-submodular function is concave along increasing directions. 
Efficient algorithms are available for maximizing DR-submodular functions over convex sets~\citep{calinescu2011maximizing,feldman11,bian2017guaranteed}. Specifically, we take $\Xs$ to be the convex hull of the indicator vectors of feasible sets.
The robust continuous optimization problem we wish to solve is then
\begin{align}
\max_{x \in \Xs} \min_{p \in \Ps_{\rho,n}} \sum\nolimits_{i =1}^n p_i F_i(x).\label{problem:continuous}
\end{align}
It remains to address two questions: (1) how to efficiently solve Problem~\eqref{problem:continuous} -- existing algorithms only apply to the max, not the maximin version -- and (2) how to then obtain a solution for Problem~\eqref{eq:robust-objective}.

We address the former question in the next section. 
For the latter question,
given a maximizing $x$ for a fixed $F$, existing techniques (e.g., swap rounding) can be used to round $x$ to a deterministic subset $S$ with no loss in solution quality~\citep{chekuri2010dependent}. But the minimax equilibrium strategy that we wish to approximate is an arbitrary distribution over subsets. Fortunately, we can show that
\begin{lemma} \label{lemma:round}
Suppose $x$ is an $\alpha$-optimal solution to Problem~\eqref{problem:continuous}. The variable $x$ induces a distribution $\Ds$ over subsets so that $\Ds$ is $(1-1/e)\alpha$-optimal for Problem~\eqref{eq:robust-objective}.
\end{lemma}
Our proof involves the \emph{correlation gap}~\citep{agrawal2010correlation}.
It is also possible to eliminate the $(1-1/e)$ gap altogether by using multiple copies of the decision variables to optimize over a more expressive class of distributions~\citep{wilder2018equilibrium}, but empirically we find this unnecessary. 

Next, we address algorithms for solving Problem~\eqref{problem:continuous}. Since a convex combination of submodular functions is still submodular, we can see each $p$ as inducing a submodular function. Therefore, in solving Problem~\eqref{problem:continuous}, we must maximize the minimum of a set of continuous submodular functions. 

\paragraph{Frank-Wolfe algorithm and complications.} 
In the remainder of this section,
we show how Problem~\eqref{problem:continuous} can be solved with optimal approximation ratio (as in Lemma~\ref{lemma:round}) by Algorithm~\ref{alg:DRO-MFW}, which is based on Frank-Wolfe (FW)~\citep{frank_algorithm_1956,jaggi_revisiting_2013}.
FW algorithms iteratively move toward the feasible point that maximizes the inner product with the gradient. Instead of a projection step, each iteration uses a single linear optimization over the feasible set $\Xs$; this is very cheap for the feasible sets we are interested in (e.g., a simple greedy algorithm for matroid polytopes).
Indeed, FW is currently the best approach for maximizing DR-submodular functions in many settings. \matt{cite?}
Observe that, since the pointwise minimum of concave functions is concave, the robust objective $G(x) = \min_{p\in\Ps_{\rho,n}} \sum_{i=1}^n p_i F_i(x)$ is also DR-submodular.
However, naive application of FW to 
$G(x)$ 
runs into several difficulties:

\begin{figure}
\begin{algorithm}[H]
    \caption{Momentum Frank-Wolfe (MFW) for DRO}
    \label{alg:DRO-MFW}
\begin{algorithmic}[1]
    \STATE {\bfseries Input:} functions $F_i$, time $T$, batch size $c$, parameter $\rho$, stepsizes $\rho_t > 0$
    \STATE $x^{(0)} \leftarrow \mathbf 0$
    \FOR{$t = 1,\dots,T$}
    	\STATE $p^{(t)} \leftarrow \underset{p \in \Ps_{\rho,n}}{\argmin} \sum_{i=1}^n p_i F_i(x^{(t-1)})$
    	\STATE Draw $i_1,\dots,i_c$ from $\{1,\dots,n\}$
    	\STATE $\tilde \nabla^{(t)} \leftarrow \frac{1}{c} \sum_{\ell=1}^c p^{(t)}_{i_\ell} \nabla F_{i_\ell}(x^{(t-1)})$
        \STATE $d^{(t)} \leftarrow (1-\rho_t) d^{(t-1)} + \rho_t \tilde \nabla^{(t)}$
    	\STATE $v^{(t)} \leftarrow \argmax_{v \in \Xs} \langle  d^{(t)}, v \rangle$
    	\STATE $x^{(t)} \leftarrow x^{(t-1)} + v^{(t)} / T$
    \ENDFOR
    \RETURN $x^{(T)}$
\end{algorithmic}
\end{algorithm}
\end{figure}

\matt{reference Algorithm~\ref{alg:stochastic-equator}}

First, 
to evaluate and differentiate $G(x)$,
we require an exact oracle for the inner minimization problem over $p$, whereas past work~\citep{Namkoong2017} gave only an approximate oracle. The issue is that two solutions to the inner problem can have arbitrarily \emph{close solution value} while also providing arbitrarily \emph{different gradients}. Hence, gradient steps with respect to an approximate minimizer may not actually improve the solution value. To resolve this issue, we provide an \emph{exact} $O(n\log n)$ time subroutine in Appendix~\ref{appendix:chi-squared-linear-oracle}, removing the $\epsilon$ of loss present in previous techniques~\citep{Namkoong2017}. Our algorithm rests on a more precise characterization of solutions to linear optimization over the $\chi^2$ ball, which is often helpful in deriving structural results for general DRO problems (e.g., Lemmas \ref{lem:arbitrary_var} and \ref{lem:smooth-and-lipschitz-variance}). 

Second, especially when the amount of data is large, we would like to use stochastic gradient estimates instead of requiring a full gradient computation at every iteration. This introduces additional noise and standard Frank-Wolfe algorithms will require $O(1/\epsilon^2)$ gradient samples per iteration to cope. Accordingly, we build on a recent algorithm of~\citet{Mokhtari2017conditional} that accelerates Frank-Wolfe by re-using old gradient information; we refer to their algorithm as \emph{Momentum Frank-Wolfe (MFW)}. For smooth DR-submodular functions, MFW achieves a $(1-1/e)$-optimal solution with additive error $\epsilon$ in $O(1/\epsilon^3)$ time. We generalize MFW to the DRO problem by solving the next challenge.

Third, Frank-Wolfe (and MFW) require a smooth objective with Lipschitz-continuous gradients; this does \emph{not} hold in general for pointwise minima. \citet{wilder2018equilibrium} gets around this issue in the context of other robust submodular optimization problems by replacing $G(x)$ with the stochastically smoothed function $G_\mu(x) = \E_{z \sim \mu}[G(x+z)]$ as in~\citep{duchi2012randomized,lan2013complexity}, where $\mu$ is a uniform distribution over a ball of size $u$. Combined with our exact inner minimization oracle, this yields a $(1-1/e)$ optimal solution to Problem~\eqref{problem:continuous} with $\epsilon$ error using $O(1/\epsilon^4)$ stochastic gradient samples. But this approach results in poor empirical performance for the DRO problem (as we demonstrate later). We obtain faster convergence, in both theory and practice, through a better characterization of the DRO problem.

\paragraph{Smoothness of the robust problem.}
While general theoretical bounds rely on smoothing $G(x)$, in practice, MFW without any smoothing performs the best. This behavior suggests that for real-world problems, the robust objective $G(x)$ may actually be smooth with Lipschitz-continuous gradient. 
Via our exact characterization of the worst-case distribution, we can make this intuition rigorous:
\begin{lemma}
\label{lem:smooth-and-lipschitz-variance}
Define $h(z) = \min_{p \in \Ps_{\rho,n}} \langle z, p \rangle$, for $z \in [0,B]^n$, and let $s_n^2$ be the sample variance of $z$.
On the subset of $z$'s satisfying the high sample variance condition $s_n^2 \geq (2\rho B^2)/n$, $h(z)$ is smooth and has $L$-Lipschitz gradient with constant $L \leq \frac{2\sqrt{2\rho}}{n^{3/2}} + \frac{2}{B n}$.
\end{lemma}
Combined with the smoothness of each $F_i$, this yields smoothness of $G$.

\begin{corollary} \label{corollary:dro-smooth}
Suppose each $F_i$ is $L_{F}$-Lipschitz. Under the high sample variance condition, $\nabla G$ is $L_G$-Lipschitz for $L_G = L_F + \frac{2b\sqrt{2\rho|V|}}{n} + \frac{2b\sqrt{|V|}}{B\sqrt{n}}$.
\end{corollary}

For submodular functions, $L_F \leq b \sqrt{k}$, where $b$ is the largest value of a single item~\citep{Mokhtari2017conditional}. However, Corollary \ref{corollary:dro-smooth} is a general property of DRO (not specific to the submodular case), with broader implications. For instance, in the convex case, we immediately obtain a $O(1/\epsilon)$ convergence rate for the gradient descent algorithm proposed by \citet{Namkoong2017} (previously, the best possible bound would be $O(1/\epsilon^2)$ via nonsmooth techniques). Our result follows from more general properties that guarantee smoothness with fewer assumptions (see Appendices~\ref{appendix:chi-square-linopt-uniqueness},~\ref{appendix:lipschitz-gradient}). For example:
\begin{fact}
For $\rho \leq \frac12$, the robust objective $h(z) = \underset{p \in \Ps_{\rho,n}}{\min} \langle z, p \rangle$ is smooth when $\{z_i\}$ are not all equal.
\end{fact}
Combined with reasonable assumptions on the distribution of $F_i$, this means $G(x)$ is nearly always smooth. Native smoothness of the robust problem yields a significant runtime improvement over the general minimum-of-submodular case. 
In particular, instead of $O(1/\epsilon^4)$, we achieve the same $O(1/\epsilon^3)$ rate of the simpler, non-robust submodular maximization:
\begin{theorem} \label{lemma:mfw-dro}
When the high sample variance condition holds, MFW with no smoothing satisfies 
\begin{align*}
\E[G(x^{(T)})] \geq \left(1 - 1/e\right)OPT - \frac{2\sqrt{kQ}}{T^{1/3}} - \frac{Lk}{T}
\end{align*}
where $Q = \max\{9^{2/3} \lVert \nabla G(x^0) - d^0 \rVert, 16\sigma^2 + 3L_G^2 k\}$; $\sigma$ is the variance of the stochastic gradients.
\end{theorem}
This convergence rate for DRO is indeed almost the same as that for a single submodular function (non-robust case)~\citep{Mokhtari2017conditional}; only the Lipschitz constant is different, but this gap vanishes as $n$ grows.

\paragraph{Comparison with previous algorithms}
Two recently proposed algorithms for robust submodular maximization could also be used in DRO, but have drawbacks compared to MFW.
Here, we compare their theoretical performance with MFW (we also show how MFW improves empirically in Section~\ref{sec:experiments}). 

First, \citet{chen2017robust} view robust optimization as a zero-sum game and apply no-regret learning to compute an approximate equilibrium. Their algorithm applies online gradient descent from the perspective of the adversary, adjusting the distributional parameters $p$. At each iteration, an $\alpha$-approximate oracle for submodular optimization (e.g., the greedy algorithm or a Frank-Wolfe algorithm) is used to compute a best response for the maximizing player. In order to achieve an $\alpha$-approximation up to additive loss $\epsilon$, the no-regret algorithm requires $O(1/\epsilon^2)$ iterations. However, each iteration requires a full invocation of an algorithm for submodular maximization. Our MFW algorithm requires runtime close to a \emph{single} submodular maximization call. This results in substantially faster runtime to achieve the same solution solution quality, as we demonstrate experimentally.

Second, \citet{wilder2018equilibrium} proposes the EQUATOR algorithm, which also applies a Frank-Wolfe approach to the multilinear extension but uses randomized smoothing as discussed earlier. Our analysis shows smoothing is unnecessary for the DRO problem, allowing our algorithm to converge using $O(1/\epsilon^3)$ stochastic gradients, while EQUATOR requires $O(1/\epsilon^4)$. This theoretical gap is reflected in empirical performance: EQUATOR converges much slower, and to lower solution quality, than MFW.   


\section{Experiments}
\label{sec:experiments}
To probe the strength and practicality of our methods, we empirically study the two motivating problems from Section~\ref{sec:math}: influence maximization and facility location.

\subsection{Facility Location}
Similar to~\citep{Mokhtari2017conditional} we consider a facility location problem motivated by recommendation systems. 
We use a music dataset from last.fm~\citep{lastfm} \sj{does it have a name?} with roughly 360000 users, 160000 bands, and over 17 million total records. For each user $i$, record $r^{i}_{j}$ indicates how many times they listened to a song by band $j$. We aim to choose a subset of bands so that the average user likes at least one of the bands, as measured by the playcounts. More specifically, we fix a collection of bands, and observe a \emph{sample} of users; we seek a subset of bands that performs well on the \emph{entire population} of users. Here, we randomly sample a subset of 1000 ``train'' users from the dataset, solve the DRO and ERM problems for $k$ bands, and evaluate performance on the remaining $\approx 360000$ ``test'' users from the dataset.

\matt{consider putting the following paragraph in the appendix}

\textbf{Optimization.}
We first compare MFW to previously proposed robust optimization algorithms, applied to the DRO problem with $k=3$. Figure~\ref{subfig:bargraph} compares 
\textbf{1.} MFW, 
\textbf{2.} Frank-Wolfe (FW) with no momentum and  
\textbf{3.} EQUATOR, proposed by~\citet{wilder2018equilibrium}.
Naive FW handles noisy gradients poorly (especially with small batches), while EQUATOR underperforms since its randomized smoothing is not necessary for our natively smooth problem. 
We also compared to the online gradient descent (OGD) algorithm of~\citet{chen2017robust}. OGD achieved slightly lower objective value than MFW with an order of magnitude greater runtime: OGD required 53.23 minutes on average, compared to 4.81 for MFW. EQUATOR and FW had equivalent runtime to MFW since all used the same batch size and number of iterations. Hence, MFW dominates the alternatives in both runtime and solution quality.

\textbf{Generalization.}
Next, we evaluate the effect of DRO on test set performance across varying set sizes $k$. Results are averaged over 64 trials for $\rho=10$ (corresponding to probability of failure $\delta = e^{-10}$ of the high probability bound).
In Figure~\ref{subfig:percent-improvement} we plot the mean percent improvement in test objective of DRO versus optimizing the average. DRO achieves clear gains, especially for small $k$.
In Figure~\ref{subfig:lastfm-test-variance} we show the variance of test performance achieved by each method. DRO achieves lower variance, meaning that overall DRO achieves better test performance, and with better consistency.


\matt{
\begin{figure}
	\centering
	\includegraphics[width=1.6in]{lastfm_train_k3_samples1000}
	\includegraphics[width=1.6in]{lastfm_test_k3_samples1000}
	\caption{Facility location on last.fm dataset. Test objective of DRO is higher, and better predicted by train objective than ERM.}\label{fig:lastfm}
\end{figure}}


\begin{figure}
	\centering
	\begin{subfigure}[b]{3.2in}
		\centering
		\includegraphics[width=3.2in]{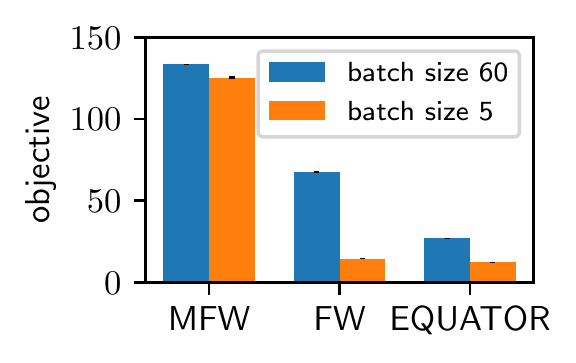}
		\caption{Algorithm comparison}
		\label{subfig:bargraph}
	\end{subfigure}
	\begin{subfigure}[b]{3.2in}
		\centering
		\includegraphics[width=3.2in]{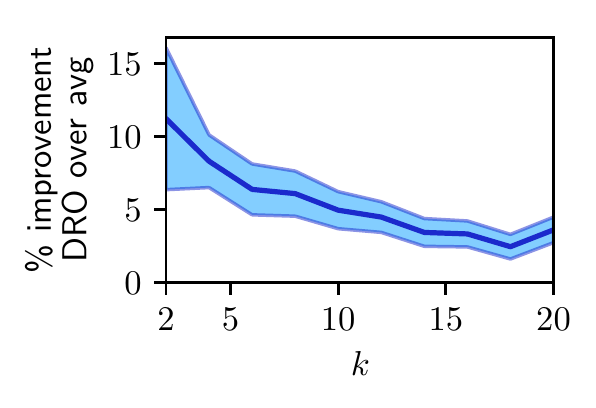}
		\caption{\% improvement via DRO}
		\label{subfig:percent-improvement}
	\end{subfigure}
	\begin{subfigure}[b]{3.2in}
		\centering
		\includegraphics[width=3.2in]{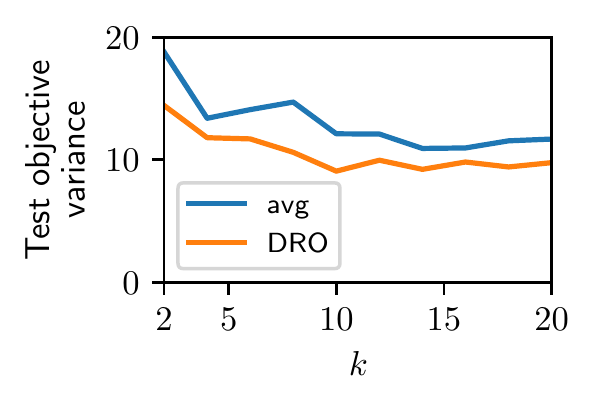}
		\caption{Test performance variance}
		\label{subfig:lastfm-test-variance}
	\end{subfigure}
	\caption{Algorithm comparison and generalization performance on last.fm dataset.}
	\label{fig:fw-convergence}
\end{figure}

\subsection{Influence maximization}
As described in Section~\ref{sec:math}, we study an influence maximization problem where we observe samples of live-edge graphs $\Es_1,\dots,\Es_n \sim P$. 
Our setting is challenging for learning: the number of samples is small and $P$ has high variance. Specifically, we choose $P$ to be a mixture of two different independent cascade models (ICM). In the ICM, each edge $e$ is live independently with probability $p_{e}$. In our mixture, each edge has $p_{e} = 0.025$ with probability $q$ and $p_{e} = 0.1$ with probability $1 - q$, mixing between settings of low and high influence spread. This models the realistic case where some messages are shared more widely than others. The mixture is \emph{not} an ICM, as observing the state of one edge gives information about 
the propagation probability 
for other edges. Handling such cases is an advantage of our DRO approach over ICM-specific robust influence maximization methods~\citep{chen_robust_2016}.

We use the political blogs dataset, a network with 1490 nodes representing links between blogs related to politics~\citep{adamic2005political}. 
Figure~\ref{fig:polblogs} compares the performance of DRO and ERM. Figure~\ref{subfig:polblogs-test-influence} shows that DRO generalizes better, achieving higher performance on the test set. Each algorithm was given $n = 20$ training samples, $k = 10$ seeds, and we set $q$ (the frequency of low influence) to be 0.1. Test influence was evaluated via a held-out set of 3000 samples from $P$. Figure~\ref{subfig:polblogs-rare-class} shows that DRO's improved generalization stems from greatly improved performance on the rare class in the mixture (low propagation probabilities). For these instances, DRO obtains a greater than \emph{40\% improvement over ERM} in held-out performance for $q = 0.1$. As $q$ increases (i.e., the rare class becomes less rare), ERM's performance on these instances converges towards DRO. A similar pattern is reflected in Figure~\ref{subfig:polblogs-test-variance}, which shows the variance in each algorithm's influence spread on the test set as a function of the number of training samples. \emph{DRO's variance is lower by 25-40\%}. As expected, DRO's advantage is greatest for small $n$, the most challenging setting for learning. 

\begin{figure}
	\centering
	\begin{subfigure}[b]{3.2in}
		\centering
		\includegraphics[width=3.2in]{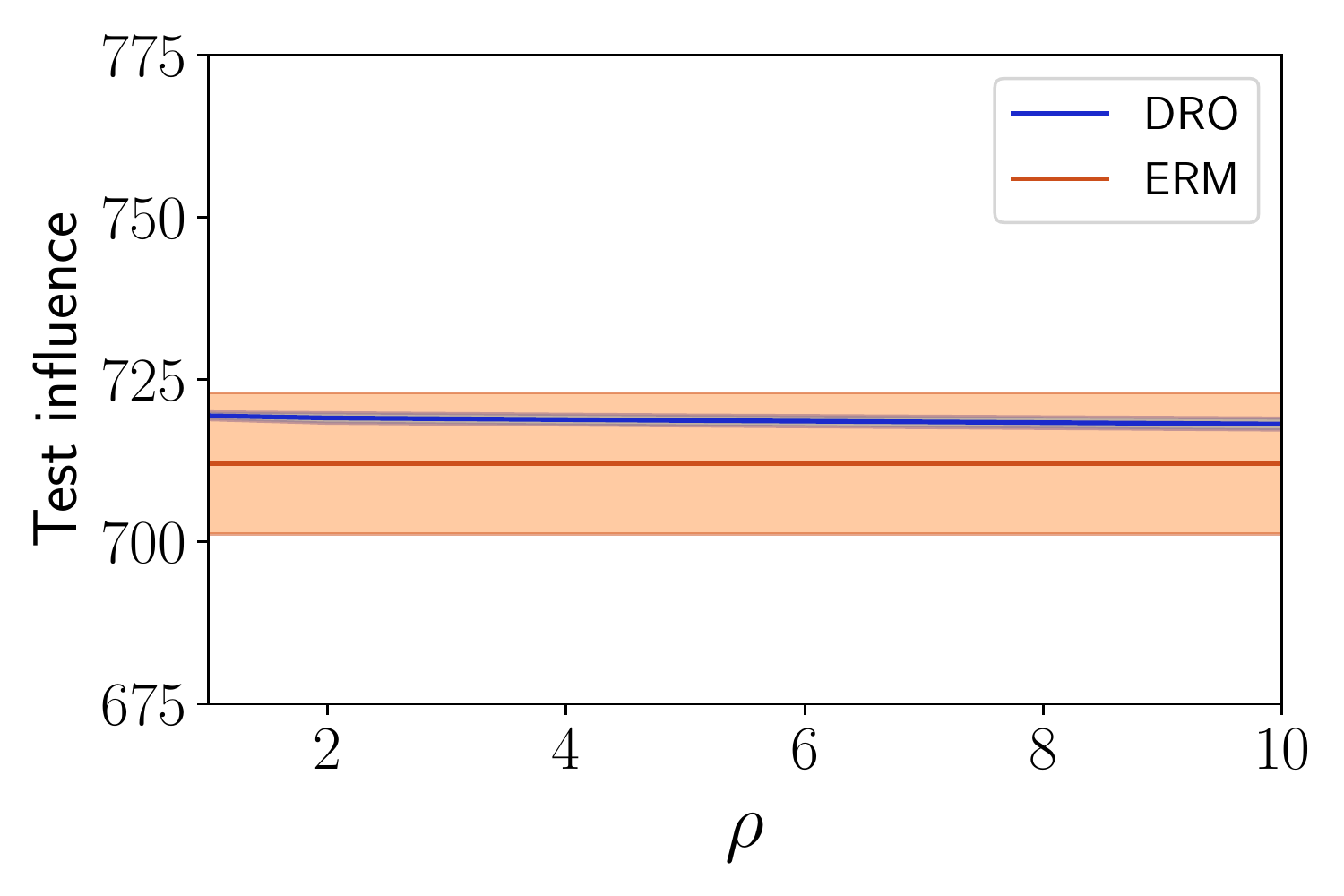}
		\caption{Influence on held-out set}
		\label{subfig:polblogs-test-influence}
	\end{subfigure}
	\begin{subfigure}[b]{3.2in}
		\centering
		\includegraphics[width=3.2in]{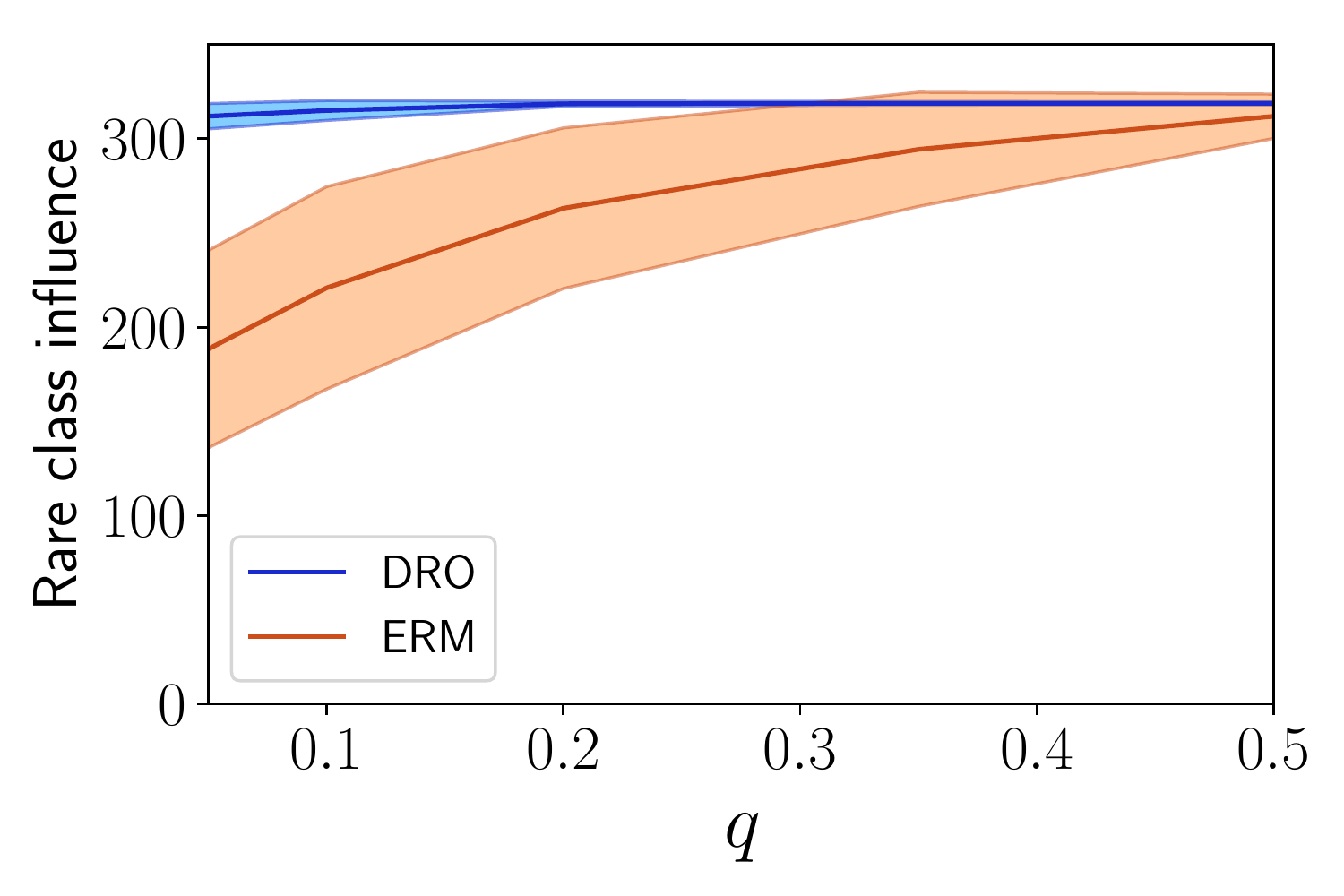}
		\caption{Rare class (held-out) influence}
		\label{subfig:polblogs-rare-class}
	\end{subfigure}
	\begin{subfigure}[b]{3.2in}
		\centering
		\includegraphics[width=3.2in]{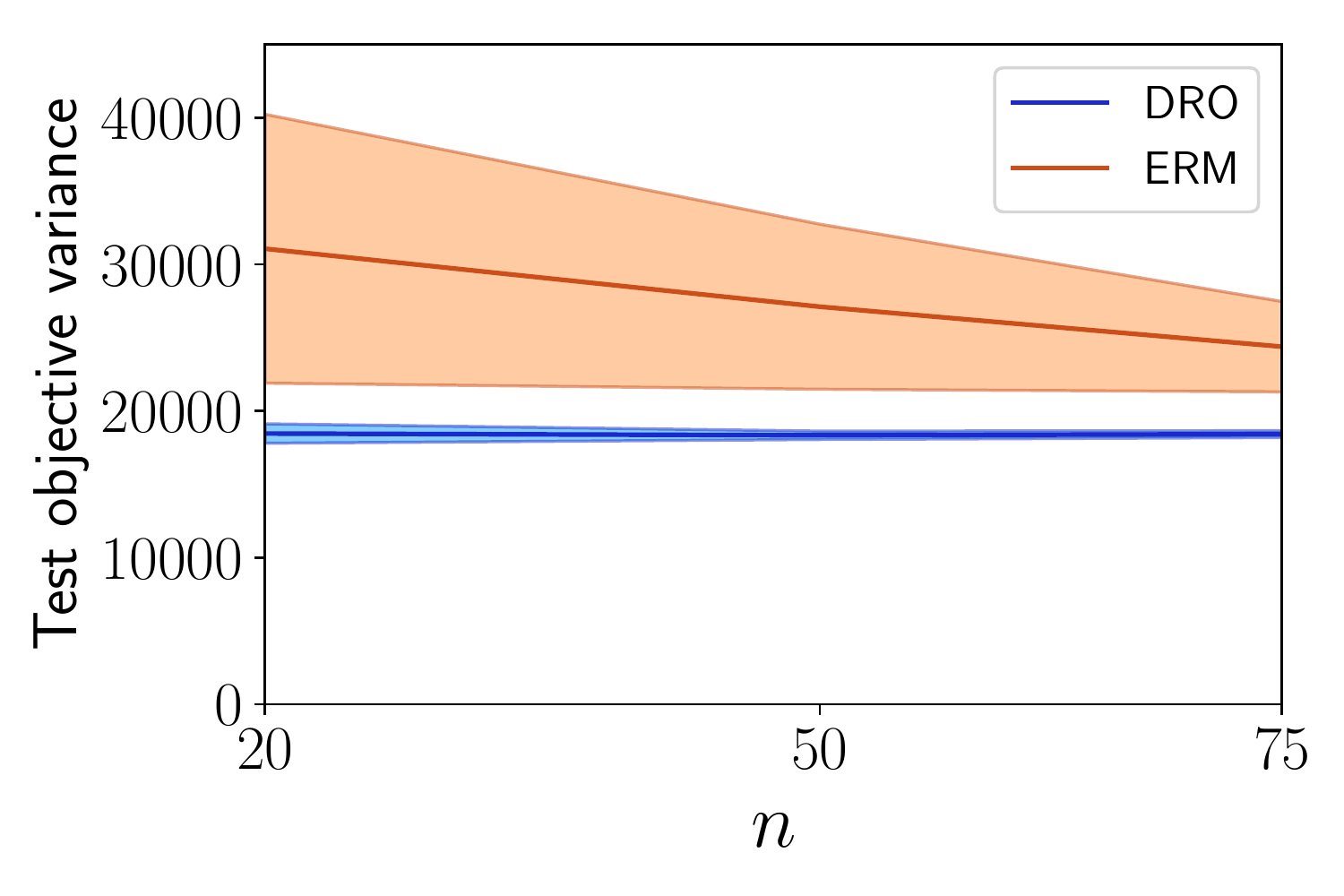}
		\caption{Test performance variance}
		\label{subfig:polblogs-test-variance}
	\end{subfigure}
	\caption{Influence maximization on political blogs dataset. 
	}\label{fig:polblogs}
\end{figure}


\section{Conclusion}
We address optimization of \emph{stochastic submodular functions} $f_P(S) = \E_P[f(S)]$ in the setting where only finite samples $f_1,\dots,f_n \sim P$ are available. 
Instead of simply maximizing the empirical mean $\frac1n\sum_i f_i$, we directly optimize a variance-regularized version which \textbf{1.} gives a high probability lower bound for $f_P(S)$ (generalization) and \textbf{2.} allows us to trade off bias and variance in estimating $f_P$. We accomplish this via an equivalent reformulation as a \emph{distributionally robust} submodular optimization problem, and show new results for the relation between distributionally robust optimization and variance regularization. Even though robust submodular maximization is hard in general, 
we are able to give efficient approximation algorithms for our reformulation. Empirically, our approach yields notable improvements for influence maximization and facility location problems.

\matt{future work?}

\subsubsection*{Acknowledgements}
This research was conducted with Government support under and awarded by DoD, Air Force Office of Scientific Research, National Defense Science and Engineering Graduate (NDSEG) Fellowship, 32 CFR 168a, and NSF Graduate Research Fellowship Program (GRFP).
This research was partially supported by The Defense Advanced Research Projects Agency (grant number YFA17 N66001-17-1-4039). The views, opinions, and/or findings contained in this article are those of the author and should not be interpreted as representing the official views or policies, either expressed or implied, of the Defense Advanced Research Projects Agency or the Department of Defense.

\bibliographystyle{plainnat}
\bibliography{ref}

\appendix

\section{Tail Bound} \label{appendix:bias-variance}

We use the following one-sided Bernstein's inequality:

\begin{lemma}[\citet{wainwrightbook}, Chapter 2] \label{lemma:bernstein}
Let $X_1...X_n$ be iid realizations of a random variable $X$ which satisfies $X \leq B$ almost surely. We have 
\begin{align*}
\Pr\left[\frac{1}{n}\sum_{i = 1}^n X_i - \E[X] \geq \epsilon \right] \leq \exp\left(-\frac{n\epsilon^2}{\Var\left(X\right) + \frac{B\epsilon}{3}}\right)
\end{align*}
\end{lemma}

We apply Lemma \ref{lemma:bernstein} with $X_i = f_i(S)$. If we set the probability on the right hand side to be at most $\delta$, then a simple calculation shows that it suffices to have $n = \frac{\Var(X)}{\epsilon^2}\log\frac{1}{\delta} + \frac{B\epsilon}{3}\log\frac{1}{\delta}$. Hence, for a given value of $n$, we can guarantee error of at most

\begin{align*}
\epsilon = \sqrt{2\log\left(\frac{1}{\delta}\right)\frac{\Var(X)}{n}} + \frac{2}{3}\log\left(\frac{1}{\delta}\right)\frac{B}{n}.
\end{align*}

Therefore, we can take $C_1 = \sqrt{2 \log\frac{1}{\delta}}$ and $C_2 = \frac{2B}{3}\log\frac{1}{\delta}$.  $B$ is often bounded in terms of the problem size for natural submodular maximization problems. For instance, for influence maximization problems we always have $B \leq \lvert V \rvert$ (though tighter bounds may be available for specific graphs and distributions).

\section{Equivalence of Variance Regularization and Distributionally Robust Optimization} \label{appendix:dro-variance-equivalence}

\begin{lemma} \label{lemma:lipschitz}
Suppose that $f(\{i\}) \leq b$ for all $f$ in the support of $P$ and all $i \in V$. Then, for each such $f$, its multilinear extension $F$ is $b$-Lipschitz in the $\ell_1$ norm. 
\end{lemma}
\begin{proof}
Consider any two points $x, x' \in [0,1]^{|V|}$ and any function $f \in \text{support}(P)$. Without loss of generality, let $f(x') \geq f(x)$. Let $[x]^+ = \max(x, 0)$ elementwise, $\lor$ denote elementwise minimum, and $1_i$ be the vector with a 1 in coordinate $i$ and zeros elsewhere. We bound $F(x')$ as 

\begin{align*}
F(x') &\leq F(x' \lor x)\\
&= F(x + [x' - x]^+)\\
&\leq F(x) + F([x' - x]^+)\\
&\leq F(x) + \sum_{i = 1}^{|V|} F([x' - x]^+_i 1_i)\\
&= F(x) + \sum_{i = 1}^{|V|} f(\{i\}) [x' - x]^+_i\\
&\leq F(x) + b\sum_{i = 1}^{|V|} [x' - x]^+_i\\
&\leq F(x) + b \lVert x' - x \rVert_1.
\end{align*}
Here, the first inequality follows from monotonicity, while the third and fourth lines use the fact that submodular functions are subadditive, i.e., $F(x + y) \leq F(x) + F(y)$. Now rearranging gives $|F(x') - F(x)| \leq b \lVert x - x' \rVert_1$ as desired.
\end{proof}

We will use the following concentration result for the sample variance of a random variable:
\begin{lemma}[\citet{Namkoong2017}, Section A.1] \label{lemma:variance-concentration}
	Let $Z$ be a random variable bounded in $[0, B]$ and $z_1...z_n$ be iid realizations of $Z$ with $n \geq 64$. Let $\sigma$ denote $\Var(Z)$ and $s_n$ denote the sample variance. It holds that $s_n^2 \geq \frac{1}{4}\sigma^2$ with probability at least $1 -  \exp\left(-\frac{n\sigma^2}{36B^2}\right)$. 
\end{lemma}

This allows us to get a uniform result for the variance expansion of the distributionally robust objective:

\begin{corollary}
Let $\Xs$ be the polytope $\{x \in [0,1]^{|V|} : \sum_{i=1}^{|V|} x_i = k\}$ corresponding to the $k$-uniform matroid.  With probability at least $1 - \delta$, for all $x \in \Xs$ such that 
\begin{equation*}
	\Var_{\Ds}(F(x)) \geq \frac{\max \{\sqrt{\frac{32}{7}\rho B^2}, \sqrt{36 B^2\left(\log\left(\frac{1}{\delta}\right) + \lvert V\rvert \log \left(1 + 24 k\right)\right)} \}}{\sqrt{n}},
\end{equation*} 
the variance expansion holds with equality.
\end{corollary}
\begin{proof}
	
	Let $\Xs_{\geq \tau} = \{x : \Var_{\Ps}(F(x)) \geq \tau\}$ be the set of points $x$ with variance at least $\tau$. Let $\Ys$ be a minimal $\ell_1$-cover of $\Xs_{\geq \tau}$ with fineness $\frac{\epsilon}{b}$, for a parameter $\epsilon$ to be fixed later. Since the $\ell_1$-diameter of $\Xs$ is $2k$ (by definition), we know that $|\Ys| \leq \left(1 + \frac{2kb}{\epsilon}\right)^{|V|}$. Let $s_n(x)$ be the sample variance of $F_1(x),\dots,F_n(x)$ and $\sigma(x) = \Var_{\Ps}(F(x))$. Via Lemma \ref{lemma:variance-concentration} and union bound, we have 
	
	\begin{align*}
	\Pr \left[s_n^2(x) \geq \frac{1}{4}\sigma^2(x) \,\, \forall x \in \Ys \right] \geq 1 - |\Ys|\exp\left(-\frac{n\tau^2}{36B^2}\right). 
	\end{align*}
	
	Conditioning on this event, we now extend the sample variance lower bound to the entirety of $\Xs_{\geq \tau}$. Consider any $x \in \Xs_{\geq \tau}$ and let $x' \in \arg\min_{x' \in \Ys} \lVert x - x'\rVert_1$. By definition of $\Ys$, $\lVert x - x\rVert_1 \leq \frac{\epsilon}{b}$, and so by Lemma \ref{lemma:lipschitz}, which guarantees Lipschitzness of each $F_i$, we have $\lvert F_i(x) - F_i(x')\rvert \leq \epsilon$ for all $i = 1,\dots,n$. Accordingly, it can be shown \bryan{fill in} that $\lvert s_n(x) - s_n(x')\rvert \leq \epsilon$ and $\lvert\sigma(x) - \sigma(x')\rvert \leq \epsilon$. Therefore, we have $s_n(x) \geq s_n(x') - \epsilon \geq \frac{1}{2} \sigma(x') - \epsilon \geq \frac{1}{2} \sigma(x) - \frac{3}{2}\epsilon$. Now by setting $\epsilon = \frac{\tau}{24}$ we have that (conditioned on the above event), $s_n(x) \geq \frac{7}{16}\tau$. Now suppose that we would like the exact variance expansion to hold on all elements of $\Xs_{\geq \tau}$ with probability at least $1 - \delta$. To have sufficiently high population variance, we must take $\tau \geq \sqrt{\frac{16}{7} \cdot \frac{2\rho B^2}{n}}$. In order for the concentration bound to hold, a simple calculation shows that $\tau \geq \sqrt{\frac{36 B^2\left(\log\left(\frac{1}{\delta}\right) + |V| \log \left(1 + 24 k\right)\right)}{n}}$ suffices. Taking the max, we need $\tau \geq \frac{\max\{\sqrt{\frac{32}{7}\rho B^2}, \sqrt{36 B^2\left(\log\left(\frac{1}{\delta}\right) + |V| \log \left(1 + 24 k\right)\right)}\}}{\sqrt{n}}$.

\end{proof}

\section{Exact Linear Oracle}
\label{appendix:chi-squared-linear-oracle}
In this section we show how to construct a $O(n\log n)$ time \emph{exact} oracle for linear optimization in the $\chi^2$ ball:
\begin{equation}
\label{eq:linopt-chi-squared}
\begin{array}{ll}
	\min_{p} & \langle z, p \rangle \\
	\text{s.t.} & \frac12 \lVert np - \mathbf 1 \rVert_2^2 \leq \rho \\
	& \mathbf 1^T p = 1 \\
	& p_i \geq 0, \; i=1,\dots,n.
\end{array}
\end{equation}
Without loss of generality, assume $z_1 \leq z_2 \leq \dots \leq z_n$. This can be done by sorting in $O(n \log n)$ time.

First, we wish to discard the case where the $\chi^2$ constraint is not tight. Let $k$ be the largest integer so that $z_1 = z_k$, i.e. $z_1 = \dots = z_k < z_{k+1}$. If it is feasible, it is optimal to place all the mass of $p$ on the first $k$ coordinates. In particular, the assignment $p_i = 1/k$ for $i=1,\dots,k$ accomplishes this while minimizing the $\chi^2$ cost. The cost can be computed as
\begin{align}
	\frac12 \sum_{i=1}^k \left( \frac{n}{k} - 1 \right)^2 + \frac12 \sum_{i=k+1}^n (0-1)^2
	&= \frac12 \left[ k \cdot \left( \frac{n-k}{k} \right)^2 + (n-k) \right] \\
	&= \frac12 \cdot (n-k) \cdot \left[ \frac{n-k}{k} + 1 \right] \\
	&= n(n-k) / (2k).
\end{align}
Hence if $\rho \geq n(n-k) / (2k)$ we can terminate immediately. Otherwise, we know the $\chi^2$ constraint must be tight.

Before proceeding, we define several auxiliary variables which can all be computed from the problem data in $O(n)$ time:
\begin{align}
	\overline z_j &= \sum_{i=1}^j z_i, \; j=1,\dots,n \\
	b_j &= \sum_{i=1}^j z_i^2, \; j=1,\dots,n \\
	s_j^2 &= \frac{b_j}{j} - (\overline z_j)^2, \; j=1,\dots,n.
\end{align}
Note that $\overline z_j$ and $s_j^2$ are the mean and variance of $\{z_1,\dots,z_j\}$.


We begin by writing down the Lagrangian of problem~\eqref{eq:linopt-chi-squared}:
\begin{equation}
	\Ls(p, \lambda, \theta, \eta) = \langle z, p \rangle + \lambda \left(\frac12 \lVert np - \mathbf 1 \rVert_2^2 - \rho\right) + \theta \left(\sum_{i=1}^n p_i - 1 \right) - \langle \eta, p \rangle,
\end{equation}
with dual variables $\lambda \in \reals_+$, $\theta \in \reals$, and $\eta \in \reals^n_+$. By KKT conditions we have
\begin{equation}
\label{eq:first-order-kkt}
	0 = \nabla_p \Ls(p, \lambda, \theta, \eta) = z + \lambda n (np - \mathbf 1) + \theta \mathbf 1 - \eta.
\end{equation}
Equivalently,
\begin{equation}
	\lambda n^2 p_i = \lambda n - z_i - \theta + \eta_i.
\end{equation}
By complementary slackness, either $\eta_i > 0$ in which case $p_i = 0$, or we have $\eta_i = 0$ and
\begin{equation}
	\lambda n^2 p_i = \lambda n - z_i - \theta.
\end{equation}
Since $z_1 \leq \dots \leq z_n$, it follows that $p_i$ decreases as $i$ increases until eventually $p_i = 0$. Hence there exists $m$ so that for $i=1,\dots,m$ we have $p_i > 0$ and thereafter $p_i=0$. Solving for $p_i$, we have that: for $i=1,\dots,m$,
\begin{align}
	\label{eq:optimal-p}
	p_i &= \left( 1 - \frac{(z_i + \theta)}{\lambda n} \right) \cdot \frac1n \text{ for } i=1,\dots,m, \\
	\text{ and } p_i &= 0 \text{ otherwise.}
\end{align}
Note we can divide by $\lambda$ as we have already determined the corresponding constraint is tight (hence $\lambda > 0$).

We will search for the best choice of $m$, and then determine $p$ based on $m$.
For fixed $\lambda, m$ we now solve for the appropriate value of $\theta$. Namely, we must have $\mathbf 1^T p = 1$:
\begin{align}
	1 = \sum_{i=1}^n p_i 
	= \sum_{i=1}^m p_i 
	&= \sum_{i=1}^m \left( 1 - \frac{(z_i + \theta)}{\lambda n} \right) \cdot \frac1n.
\end{align}
Simplifying,
\begin{align}
	n = \sum_{i=1}^m \left( 1 - \frac{(z_i + \theta)}{\lambda n} \right)
	&= m - \frac{1}{\lambda n}\sum_{i=1}^m (z_i + \theta) \\
	&= m - \frac{m \overline z_m}{\lambda n} - \frac{\theta m}{\lambda n}.
\end{align}
Multiplying through by $\lambda n$ and solving for $\theta$, we have
\begin{align}
	\lambda n^2 = \lambda m n - m \overline z_m - \theta m
	\implies \theta = \left( 1 - \frac{n}{m} \right) \lambda n - \overline z_m.
\end{align}

Now that we have solved for $\theta$ as a function of $\lambda$ and $m$, the variable $p$ is purely a function of $m$ and $\lambda$.
For fixed $\lambda$ and $m$, it is not hard to compute the objective value attained by the value of $p$ induced by equation~\eqref{eq:optimal-p}:
\begin{align}
	\langle z, p \rangle &= \frac1n \sum_{i=1}^m \left( 1 - \frac{(z_i + \theta)}{\lambda n}\right) z_i \\
	&= \frac1n \sum_{i=1}^m z_i - \frac1n \sum_{i=1}^m \frac{(z_i + \theta) z_i}{\lambda n} \\
	&= \frac{m}{n} \overline z_m - \frac{1}{\lambda n^2} \sum_{i=1}^m (z_i^2 + \theta z_i ) \\
	&= \frac{m}{n} \overline z_m - \frac{1}{\lambda n^2} (b_m + \theta m \overline z_m ) \\
	&= \frac{m}{n} \overline z_m - \frac{1}{\lambda n^2} \left(b_m + \left( \left( 1 - \frac{n}{m} \right) \lambda n - \overline z_m \right) m \overline z_m \right) \\
	&= \frac{m}{n} \overline z_m - \frac{b_m}{\lambda n^2} - \frac{(1 - n/m) \lambda n m \overline z_m}{\lambda n^2} + \frac{m (z_m)^2}{\lambda n^2} \\
	&= \frac{m}{n} \overline z_m - \frac{b_m}{\lambda n^2} + \frac{(n - m) \overline z_m}{n} + \frac{m (z_m)^2}{\lambda n^2} \\
	&= \overline z_m - \frac{b_m}{\lambda n^2} + \frac{m (z_m)^2}{\lambda n^2} \\
	&= \overline z_m - \frac{1}{\lambda n^2} \cdot (b_m - m (\overline z_m)^2) \\
	&= \overline z_m - \frac{m s_m^2}{\lambda n^2}.
\end{align}
Since $m s_m^2 \geq 0$, for fixed $m$ we seek the minimum value of $\lambda$ such that the induced $p$ is still feasible. Since the $\mathbf 1^T p = 1$ constraint is guaranteed by the choice of $\theta$, we need only check the $\chi^2$ and nonnegativity constraints. In section~\ref{appendix:lambda-constraints} we derive that the optimal feasible $\lambda$ is given by
\begin{equation}
	\label{eq:optimal-lambda}
	\lambda = \frac{1}{n^2} \cdot \max\left\{ \sqrt{\frac{m^2 s_m^2}{\alpha(m,n,\rho)}}, \; m(z_m - \overline z_m) \right\}.
\end{equation}
Hence, in constant time for each candidate $m$ with $\alpha(m,n,\rho)$, we select $\lambda$ per equation~\eqref{eq:optimal-lambda} and evaluate the objective. Finally, we return $p$ corresponding to the optimal choice $m$. This algorithm is given more formally in Algorithm~\ref{alg:chi-squared-linopt}.
\begin{algorithm}[tb]
    \caption{Linear optimization in $\chi^2$ ball}
    \label{alg:chi-squared-linopt}
\begin{algorithmic}
    \STATE {\bfseries Input:} pre-sorted vector $z$ with $z_1 \leq \dots \leq z_n$
    \STATE {\bfseries Output:} optimal vector $p$
    \STATE Compute maximum $k$ s.t. $z_1 = z_k$
    \IF{$n(n-k)/(2k) \leq \rho$}
    	\RETURN{$p$ with $p_i = 1/k \cdot \mathbf 1\{i \leq k\}$}
    \ENDIF

    \COMMENT{now we must search for optimal $m$}

    \STATE $\overline z_j \leftarrow \frac1j \sum_{i=1}^j z_i, \; j=1,\dots,n$
	\STATE $b_j \leftarrow \sum_{i=1}^j z_i^2, \; j=1,\dots,n$
	\STATE $s_j^2 \leftarrow b_j/j - (\overline z_j)^2, \; j=1,\dots,n$
	\STATE $m_\text{min} \leftarrow \min\{ m \in \{1,\dots,n\} : \alpha(m,n,\rho) > 0\}$
	\STATE $\lambda_m = \frac{1}{n^2} \cdot \max\left\{ \sqrt{\frac{m^2 s_m^2}{\alpha(m,n,\rho)}}, \; (z_m - \overline z_m)m \right\}, \; m=m_\text{min},\dots,n$

    \STATE $v_m \leftarrow \overline z_m - m s_m^2 / (\lambda_m n^2), \; m=m_\text{min},\dots,n$
    \STATE $m_\text{opt} \leftarrow \argmin_m \{ v_m : m = m_\text{min},\dots,n \}$

    \STATE $\theta \leftarrow \left( 1 - \frac{n}{m_\text{opt}} \right) \lambda_{m_\text{opt}} n - \overline z_{m_\text{opt}}$
    \RETURN $p = \frac1n \max\left( 0, 1 - \frac{z_{m_\text{opt}} + \theta}{\lambda_{m_\text{opt}} n} \right)$
\end{algorithmic}
\end{algorithm}

\subsection{Constraints on $\lambda$ for fixed $m$} \label{appendix:lambda-constraints}
First we check the $\chi^2$ constraint; since $\lambda > 0$, we have: 
\begin{align}
	\rho &\geq \frac12 \lVert np - \mathbf 1 \rVert_2^2 \\
	&= \frac12 \sum_{i=1}^n (np_i - 1)^2 \\
	&= \frac12 \sum_{i=1}^m \left( \left(1 - \frac{(z_i + \theta)}{\lambda n}\right) - 1 \right)^2 + \frac12 \sum_{i=m+1}^n (-1)^2 \\
	&= \frac12 \cdot \frac{1}{\lambda^2 n^2} \sum_{i=1}^m (z_i + \theta)^2 + \frac12 (n - m).
	\label{eq:chi-squared-constraint-rhs}
\end{align}
We expand the sum of squares:
\begin{align}
	\sum_{i=1}^m (z_i + \theta)^2
	&= \sum_{i=1}^m (z_i^2 + 2z_i\theta + \theta^2) \\
	&= \sum_{i=1}^m z_i^2 + 2\theta \sum_{i=1}^m z_i + \sum_{i=1}^m \theta^2 \\
	&= b_m + 2\theta m \overline z_m + \theta^2 m.
\end{align}
Plugging in our expression for $\theta$, this equals:
\begin{align}
	b_m + 2\theta m \overline z_m + \theta^2 m
	&= b_m + 2m \overline z_m \theta + \left[ \left( 1 - \frac{n}{m} \right) \lambda n - \overline z_m \right]^2 \cdot m \\
	&= b_m + 2m \overline z_m \theta + \left[ \left( 1 - \frac{n}{m} \right)^2 \lambda^2 n^2 - 2 \left( 1 - \frac{n}{m} \right) \lambda n \cdot \overline z_m + (\overline z_m)^2 \right] \cdot m \\
	&= b_m + 2m \overline z_m \theta + \left( 1 - \frac{n}{m} \right)^2 \lambda^2 n^2 m - 2 \left( 1 - \frac{n}{m} \right) \lambda n m \overline z_m + m (\overline z_m)^2 \\
	&= b_m + 2m \overline z_m \left[ \left( 1 - \frac{n}{m} \right) \lambda n - \overline z_m \right] + \left( 1 - \frac{n}{m} \right)^2 \lambda^2 n^2 m - 2 \left( 1 - \frac{n}{m} \right) \lambda n m \overline z_m + m (\overline z_m)^2 \\
	&= b_m + 2\left( 1 - \frac{n}{m} \right) \lambda n m \overline z_m - 2 m (\overline z_m)^2 + \left( 1 - \frac{n}{m} \right)^2 \lambda^2 n^2 m - 2 \left( 1 - \frac{n}{m} \right) \lambda n m \overline z_m + m (\overline z_m)^2 \\
	&= b_m - 2m(\overline z_m)^2 + \left( 1 - \frac{n}{m} \right)^2 \lambda^2 n^2 m + m(\overline z_m)^2 \\
	&= b_m - m(\overline z_m)^2 + \left( 1 - \frac{n}{m} \right)^2 \lambda^2 n^2 m \\
	&= m s_m^2 + \left( 1 - \frac{n}{m} \right)^2 \lambda^2 n^2 m.
\end{align}
Finally, plugging this back into equation~\eqref{eq:chi-squared-constraint-rhs} yields:
\begin{align}
	\rho &\geq \frac12 \cdot \frac{1}{\lambda^2 n^2} \cdot \left[ m s_m^2 + \left( 1 - \frac{n}{m} \right)^2 \lambda^2 n^2 m \right] + \frac12 \cdot (n - m) \\
	\Leftrightarrow 2\rho &\geq \frac{m s_m^2}{\lambda^2 n^2} + \left(1 - \frac{n}{m} \right)^2 m + (n-m) \\
	\Leftrightarrow 2\rho &\geq \frac{m s_m^2}{\lambda^2 n^2} + \left(1 - \frac{2n}{m} + \frac{n^2}{m^2} \right) m + (n-m) \\
	\Leftrightarrow 2\rho &\geq \frac{m s_m^2}{\lambda^2 n^2} + m - 2n + \frac{n^2}{m} + (n-m) \\
	\Leftrightarrow 2\rho &\geq \frac{m s_m^2}{\lambda^2 n^2} - n + \frac{n^2}{m} \\
	\Leftrightarrow \frac{2\rho m}{n^2} &\geq \frac{m^2 s_m^2}{\lambda^2 n^4} - \frac{m}{n} + 1 \\
	\Leftrightarrow \frac{m^2 s_m^2}{\lambda^2 n^4} &\leq \alpha(m,n,\rho),
\end{align}
where $\alpha(m,n,\rho)$ is defined as in the main text.
If $\alpha(m,n,\rho)\leq0$, there is no feasible choice of $\lambda$ for this $m$. Otherwise, we can divide and solve for $\lambda$:
\begin{equation}
	\lambda \geq \sqrt{\frac{m^2 s_m^2}{n^4 \alpha(m,n,\rho)}}
	= \frac1n \sqrt{\frac{m s_m^2}{2\rho + n - n^2 / m}},
\end{equation}
or equivalently
\begin{equation}
	\label{eq:lambda-n-squared}
	\lambda n^2 \geq \sqrt{\frac{m^2 s_m^2}{\alpha(m,n,\rho)}}.
\end{equation}

Now we check the other remaining constraint on $\lambda$, that the constraint $p_i \geq 0$ for $i=1,\dots,m$ must hold. In particular, we must have $p_m \geq 0$:
\begin{align}
	0 \leq p_m &= \frac1n \cdot \left( 1 - \frac{z_m+\theta}{\lambda n} \right) \\
	\Leftrightarrow   z_m + \theta &\leq \lambda n \\
	\Leftrightarrow z_m + \left( 1 - \frac{n}{m} \right) \lambda n - \overline z_m &\leq \lambda n \\
	\Leftrightarrow z_m - \overline z_m &\leq \frac{\lambda n^2}{m} \\
	\Leftrightarrow m (z_m - \overline z_m) &\leq \lambda n^2.
\end{align}
Hence $\lambda$ must satisfy
\begin{equation}
	\lambda n^2 \geq \max\left\{ \sqrt{\frac{m^2 s_m^2}{\alpha(m,n,\rho)}}, \; m(z_m - \overline z_m) \right\}.
\end{equation}
Since we seek minimal $\lambda$, we select $\lambda$ which makes this constraint tight.

\subsection{Unique solutions} \label{appendix:chi-square-linopt-uniqueness}
Here we provide results for understanding when there is a unique solution to Problem~\eqref{eq:linopt-chi-squared}. Recall that our solution to Problem~\eqref{eq:linopt-chi-squared} first checks whether the optimal solutions have tight $\chi^2$ constraint. By choosing $\rho$ small enough, this can be guaranteed uniformly:
\begin{lemma}
	Suppose $\{z_i\}$ attain at least $\ell$ distinct values. If $\rho \leq (\ell - 1)/2$ then all optimal solutions to Problem~\eqref{eq:linopt-chi-squared} have tight $\chi^2$ constraint.
\end{lemma}
\begin{proof}
Assume $z_1 \leq \dots \leq z_n$.
If $\{z_i\}$ attain at least $\ell$ distinct values, then the maximum number $k$ so that $z_1 = \dots = z_k$ can be bounded by $n - \ell + 1$. Recall from earlier in section~\ref{appendix:chi-squared-linear-oracle} that the constraint is tight if $\rho \leq n(n-k)/(2k)$, and note that this bound is monotone decreasing in $k$. Hence, we can guarantee the constraint is tight as long as 
\begin{equation}
	\rho \leq \frac{n(n-(n-\ell+1))}{2(n-\ell+1)}
	= \frac{n(\ell - 1)}{2(n-\ell+1)}.
\end{equation}
Since $n - \ell + 1 \leq n$, the previous inequality is implied by
\begin{equation}
	\rho \leq 
	\frac{(n-\ell+1)(\ell - 1)}{2(n-\ell+1)}
	= \frac{\ell-1}{2}.
\end{equation}
\end{proof}

Now, assuming the $\chi^2$ constraint is tight, we can characterize the set of optimal solutions:
\begin{lemma}
Suppose the optimal solutions for Problem~\eqref{eq:linopt-chi-squared} all have tight $\chi^2$ constraint. Then there is a unique optimal solution $p^*$ with minimum cardinality among all optimal solutions.
\end{lemma}
\begin{proof}
This is a consequence of our characterization of the optimal dual variable $\lambda$ as a function of the sparsity $m$. For each choice of $m$, we solved earlier for the unique dual variable $\lambda_m$ which determines a unique solution $p$. Hence, even if there are multiple values of $m$ that are feasible and that yield optimal objective value, there is still a unique minimal $m_\text{opt}$, which in turn yields a unique optimal solution.
\end{proof}

\subsection{Lipschitz gradient} \label{appendix:lipschitz-gradient}
\begin{lemma}
Define $h(z) = \min_{p \in \Ps_{\rho,n}} \langle z, p \rangle$. 
Then on the subset of $z$'s satisfying the high sample variance condition $s_n^2 \geq (2\rho B^2)/n^2$, $h(z)$ has Lipschitz gradient with constant $L \leq \frac{2\sqrt{2\rho}}{n^{3/2}} + \frac{2}{B n^{1/2}}$.
\end{lemma}
\begin{proof}
In this regime, there is a unique worst-case $p \in \Ps_{\rho,n}$, and it is the gradient of $h(z)$. In the high sample variance regime, we have $m=n$, i.e. each $p_i > 0$ and:
\begin{align}
	p_i = \left( 1 - \frac{z_i + \theta}{\lambda n} \right) \cdot \frac1n \text{ for all } i=1,\dots,n.
\end{align}
In particular, $\theta = (1 - n/n)\lambda n - \overline z_n = -\overline z_n$, and $\lambda = \frac{1}{n^2} \sqrt{n^2 s_n^2 / (2\rho / n)}$. Simplifying, we have
\begin{align}
	p_i &= \left( 1 - \frac{z_i - \overline z_n}{\lambda n} \right) \cdot \frac1n \\
	&= \left( 1 - \frac{z_i - \overline z_n}{\frac{1}{n} \sqrt{n^2 s_n^2 / (2\rho / n)}} \right) \cdot \frac1n \\
	&= \left( 1 - \frac{z_i - \overline z_n}{\sqrt{n s_n^2 / (2\rho)}} \right) \cdot \frac1n.
\end{align}
We will bound the Lipschitz constant of $p$ as a function of $z$ by computing the Hessian which has entries $H_{ij} = \frac{\partial p_i}{\partial z_j}$ and bounding its largest eigenvalue. For the element $H_{ij}$ we have two cases. If $i=j$, then
\begin{align}
	H_{ii} &= -\frac{\sqrt{2\rho}}{n^{3/2}} \cdot \frac{\partial}{\partial z_i} \left( \frac{z_i - \overline z_n}{\sqrt{s_n^2}} \right) \\
	&= -\frac{\sqrt{2\rho}}{n^{3/2}} \cdot \left( \frac{\sqrt{s_n^2} (1 - \frac1n) - (z_i - \overline z_n) \cdot \frac2n \cdot (z_i - \overline z_n)}{s_n^2} \right).
\end{align}
If $i\not=j$, then
\begin{align}
	H_{ij} &= -\frac{\sqrt{2\rho}}{n^{3/2}} \cdot \frac{\partial}{\partial z_j} \left( \frac{z_i - \overline z_n}{\sqrt{s_n^2}} \right) \\
	&= -\frac{\sqrt{2\rho}}{n^{3/2}} \cdot \left( \frac{-\frac1n \cdot \sqrt{s_n^2} - (z_i - \overline z_n) \cdot \frac2n \cdot (z_j - \overline z_n)}{s_n^2} \right).
\end{align}
Define $\tilde H$ so that $\frac{\sqrt{2\rho}}{n^{3/2} s_n^2} \tilde H = H$, i.e.
\begin{equation}
\tilde H_{ij} = \begin{cases}
	\sqrt{s_n^2} (\frac1n - 1) + (z_i - \overline z_n) \cdot \frac2n \cdot (z_i - \overline z_n) & i=j \\
	\frac1n \cdot \sqrt{s_n^2} + (z_i - \overline z_n) \cdot \frac2n \cdot (z_j - \overline z_n) & i\not=j.
\end{cases}
\end{equation}
It is easy to see that $\tilde H$ is given by
\begin{equation}
	\tilde H = -\text{diag}(\sqrt{s_n^2} \mathbf 1) + \frac{\sqrt{s_n^2}}{n} \mathbf 1 \mathbf 1^T + \frac2n (z - \overline z_n \mathbf 1) (z - \overline z_n \mathbf 1)^T.
\end{equation}
By the triangle inequality, the operator norm of $\tilde H$ can thus be bounded by
\begin{align}
	\lVert \tilde H \rVert &\leq \lVert \text{diag}(\sqrt{s_n^2} \mathbf 1) \rVert + \frac{\sqrt{s_n^2}}{n} \lVert\mathbf 1 \mathbf 1^T \rVert +  \frac2n \lVert(z - \overline z_n \mathbf 1) (z - \overline z_n \mathbf 1)^T \rVert \\
	&= \sqrt{s_n^2} + \frac{\sqrt{s_n^2}}{n} \lVert \mathbf 1 \rVert_2^2 + \frac2n \lVert z - \overline z_n \mathbf 1 \rVert_2^2 \\
	&= 2\sqrt{s_n^2} + \frac2n \sum_{i=1}^n (z_i - \overline z_n)^2 \\
	&= 2\sqrt{s_n^2} + 2 s_n^2.
\end{align}
It follows that the Lipschitz constant of the gradient of $h(z)$ can be bounded by
\begin{align}
	\lVert H \rVert &= \frac{\sqrt{2\rho}}{n^{3/2} s_n^2} \lVert \tilde H \rVert \\
	&\leq \frac{\sqrt{2\rho}}{n^{3/2} s_n^2} \left( 2\sqrt{s_n^2} + 2 s_n^2 \right) \\
	&= \frac{2\sqrt{2\rho}}{n^{3/2}} \cdot \left( 1 + \frac{1}{\sqrt{s_n^2}}\right).
\end{align}
Since we are in the high variance regime $s_n^2 \geq (2\rho B^2)/n$, it follows that $1/\sqrt{s_n^2} \leq \sqrt{n} / (B \sqrt{2\rho})$ and therefore
\begin{align}
	\lVert H \rVert &\leq \frac{2\sqrt{2\rho}}{n^{3/2}} \cdot \left( 1 + \frac{\sqrt{n}}{B \sqrt{2\rho}} \right) \\
	&= \frac{2\sqrt{2\rho}}{n^{3/2}} + \frac{2}{B n}.
\end{align}
\end{proof}

\section{Projection onto the $\chi^2$ ball} \label{appendix:chi-square-project}
Let $w \in \reals^n$ be pre-sorted (taking time $O(n\log n)$), so that $w_1 \geq \dots \geq w_n$.
We wish to solve the problem
\begin{equation}
\label{eq:project-chi-squared}
\begin{array}{ll}
	\min_{p} & \frac12 \lVert p - w \rVert_2^2 \\
	\text{s.t.} & \frac12 \lVert np - \mathbf 1 \rVert_2^2 \leq \rho \\
	& \mathbf 1^T p = 1 \\
	& p_i \geq 0, \; i=1,\dots,n.
\end{array}
\end{equation}
As in section~\ref{appendix:chi-squared-linear-oracle}, we start by precomputing the auxiliary variables:
\begin{align}
	\overline w_j &= \sum_{i=1}^j w_i, \; j=1,\dots,n \\
	b_j &= \sum_{i=1}^j w_i^2, \; j=1,\dots,n \\
	s_j^2 &= \frac{b_j}{j} - (\overline w_j)^2, \; j=1,\dots,n.
\end{align}
We remark that these can be updated efficiently when sparse updates are made to $w$; coupled with a binary search over optimal $m$, this can yield $O(\log n)$ update time as in~\citep{Duchi2008,namkoong2016stochastic}.

We form the Lagrangian:
\begin{align}
	\Ls(p,\lambda,\theta,\eta) = \frac12 \lVert p - w \rVert_2^2 + \lambda \left( \frac12 \lVert np - \mathbf 1 \rVert_2^2 - \rho \right) + \theta \left( \sum_{i=1}^n p_i - 1 \right) - \langle \eta, p \rangle
\end{align}
with dual variables $\lambda \in \reals_+, \theta \in \reals$, and $\eta \in \reals^n_+$. We will also use the reparameterization $\beta = 1/(1 + \lambda n^2)$ throughout. By KKT conditions we have
\begin{align}
	0
	= \nabla_p \Ls(p,\lambda,\theta,\eta) 
	&= p - w + \lambda n \left( np - \mathbf 1 \right) + \theta \mathbf 1 - \eta \\
	&= (1 + \lambda n^2) p - w - \lambda n + \theta \mathbf 1 - \eta.
\end{align}
For any given $i$, if $\eta_i > 0$ we have $p_i = 0$ by complementary slackness. Otherwise, if $\eta_i=0$ we have
\begin{align}
	0 &=(1 + \lambda n^2) p_i - w_i - \lambda n + \theta\\
	\Leftrightarrow (1+\lambda n^2) p_i &= w_i + \lambda n - \theta.
\end{align}
The variable $p$ is implicitly given here by $\theta$, $\lambda$ and $m$. 
Next we seek to solve for $\theta$ as a function of $\lambda$ and $m$.

Note that since $w_i$ decreases as $i$ increases, therefore $p_i$ also decreases. It follows that for some $m \in \{1,\dots,n\}$, we have $p_i > 0$ for $i \leq m$ and $p_i = 0$ otherwise. Since $p_i$ must sum to one, we have
\begin{align}
	1 &= \sum_{i =1}^n p_i = \sum_{i =1}^m p_i \\
	\Leftrightarrow (1+\lambda n^2) &= \sum_{i=1}^m (w_i + \lambda n - \theta) \\
	&= m \overline w_m + \lambda mn - m\theta
\end{align}
from which it follows that $\theta = \overline w_m + \lambda n - (1+\lambda n^2)/m$. Plugging this into the expression for $p_i$ and rearranging yields
\begin{equation}
	p_i = (w_i - \overline w_m) \beta + 1/m.
\end{equation}

It will become apparent later that the objective improves as $\beta$ increases, and so for fixed $m$ we seek the largest $\beta$ which yields a feasible $p$.
First, we check the $\chi^2$ constraint:
\begin{align}
	\rho &\geq \frac12 \sum_{i=1}^n (np_i - 1)^2 \\
	&= \frac12 \sum_{i=1}^m (np_i - 1)^2 + \frac12 \sum_{i=m+1}^n 1^2 \\
	&= \frac12 \sum_{i=1}^m (n \beta (w_i - \overline w_m) + n/m - 1)^2 + \frac12 (n-m).
\end{align}
Expanding and multiplying by 2, we have
\begin{align}
	2 \rho 
	&\geq \sum_{i=1}^m \left[n^2\beta^2 (w_i - \overline w_m)^2 + 2n\beta(n/m - 1)(w_i - \overline w_m) + (n/m-1)^2 \right] + n-m.
\end{align}
The middle term in the sum cancels because $\sum_{i=1}^m w_i = m \overline w_m$. We are left with
\begin{align}
	2 \rho 
	&\geq n^2\beta^2 \sum_{i=1}^m (w_i - \overline w_m)^2 + m (n/m-1)^2 + n-m \\
	&= n^2\beta^2 m s_m^2 + m (n/m-1)^2 + n-m.
\end{align}
Solving for $\beta^2$, we are left with
\begin{equation}
	\beta^2 \leq \frac{2\rho + n - n^2/m}{n^2 m s_m^2}
	= \frac{2\rho m/n^2 + m/n - 1}{m^2 s_m^2}
	= \frac{\alpha(m,n,\rho)}{m^2 s_m^2},
\end{equation}
where $\alpha(m,n,\rho)$ is defined as in the main text.
This gives the maximum value of $\beta$ for which the $\chi^2$ constraint is met. We also need to check the $p_i \geq 0$ constraint. This is more straightforward: we must have
\begin{equation}
	0 \leq p_i = (w_i - \overline w_m) \beta + 1/m
\end{equation}
for all $i=1,\dots,m$. Since $w_i$ is decreasing, it suffices to check $i=m$.
If $w_m - \overline w_m \geq 0$ there is no problem, as $\beta > 0$. Otherwise, we divide and are left with the condition
\begin{equation}
	\beta \leq \frac{1}{m (\overline w_m - w_m)}.
\end{equation}
Our exact algorithm is now straightforward: for each $m$, compute the largest feasible $\beta$ (if there is a feasible $\beta$), compute the corresponding objective value, and then return $p$ corresponding to the best $m$.

If $\alpha(m,n,\rho) < 0$ for a given $m$, we can immediately discard that choice of $m$ as infeasible. Otherwise we compute $\beta$ and check the objective value $v_m$ for that $m$:
\begin{align}
	v_m &= \frac12 \lVert p - w \rVert_2^2 \\
	&= \frac12 \sum_{i=1}^m ( (w_i - \overline w_m) \beta + 1/m - w_i)^2 + \frac12 \sum_{i=m+1}^n w_i^2 \\
	&= \frac12 \sum_{i=1}^m \left[ (w_i - \overline w_m)^2 \beta^2 + 2\beta(w_i - \overline w_m) (1/m - w_i) + (1/m - w_i)^2 \right] + \frac12 \sum_{i=m+1}^n w_i^2.
\end{align}
As before, the $\sum_{i=1}^m 2\beta(w_i - \overline w_m)/m$ term cancels and we are left with
\begin{align}
	v_m &= \frac12 \sum_{i=1}^m \left[ (w_i - \overline w_m)^2 \beta^2 - 2\beta w_i(w_i - \overline w_m) + (1/m - w_i)^2 \right] + \frac12 \sum_{i=m+1}^n w_i^2 \\
	&= \frac12 \cdot \beta^2 m s_m^2 - \beta\sum_{i=1}^m w_i^2 + \beta \sum_{i=1}^m w_i \overline w_m + \frac12 \sum_{i=1}^m (1/m - w_i)^2 + \frac12 \sum_{i=m+1}^n w_i^2 \\
	&= \frac12 \cdot \beta^2 m s_m^2 - \beta b_m + \beta m (\overline w_m)^2 + \frac12 \sum_{i=1}^m (1/m^2 - 2w_i/m + w_i^2) + \frac12 (b_n - b_m) \\
	&= \frac12 \cdot \beta^2 m s_m^2 - \beta b_m + \beta m (\overline w_m)^2 + \frac12 \cdot \frac1m - \overline w_m + \frac12 b_m + \frac12 (b_n - b_m) \\
	&= \frac12 \cdot \beta^2 m s_m^2 - \beta b_m + \beta m (\overline w_m)^2 + \frac{1}{2m} - \overline w_m + \frac12 b_n \\
	&= \frac12 \cdot \frac{\alpha(m,n,\rho)}{m} - \beta (b_m -  m (\overline w_m)^2) + \frac{1}{2m} - \overline w_m + \frac12 b_n \\
	&= \left(\frac{\rho}{n^2} + \frac{1}{2n} - \frac{1}{2m} \right) - \beta m s_m^2 + \frac{1}{2m} - \overline w_m + \frac12 b_n \\
	&= \left(\frac{\rho}{n^2} + \frac{1}{2n} \right) - \beta m s_m^2 - \overline w_m + \frac12 b_n.
\end{align}
Discarding the terms which do not depend on $m$, we seek $m$ which minimizes $\tilde v_m := -\beta m s_m^2 - \overline w_m$. Finally, we remark that it is now quite apparent that for fixed $m$ we wish to maximize $\beta$.

\section{Convergence analysis for MFW}

%
%
%
 
 Here we establish the convergence rate of the MFW algorithm specifically for the DRO problem. The main work is to establish Lipschitz continuity of $\nabla G$, the gradient of the DRO objective. In fact, \citet{Mokhtari2017conditional} get a better bound by controlling changes in $\nabla G$ specifically along the updates used by MFW. We bound this same quantity as follows: 
 
 \begin{lemma}
 	When the high sample variance condition is satisfied, for any two points $x^{(t)}$ and $x^{(t+1)}$ produced by MFW, $\nabla G$ satisfies $\lVert \nabla G(x^{(t)}) - \nabla G(x^{(t+1)})\rVert \leq \left(b \sqrt{n |V|} L + b \sqrt{k}\right) \lVert x^{(t)} - x^{(t+1)} \rVert$. 
 \end{lemma}
\begin{proof}
 
 We write $\vec{F}(x) = (F_1(x), ..., F_n(x))$, and are interested in the composition $G = h(\vec{F}(x))$ (recall that $h$ is defined in Lemma \ref{lem:smooth-and-lipschitz-variance} as the value of the inner minimization problem for a given set of values). Let $D \vec{F}(x)$ be the matrix derivative of $\vec{F}$. That is, $\left[D \vec{F}(x)\right]_{ij} = \frac{\partial}{\partial x_j} F_i(x)$. The chain rule yields
 
 \begin{align*}
 \nabla h(\vec{F}(x)) = \left(\nabla h(\vec{F}(x))\right) D \vec{F}(x).
 \end{align*}
 
 Consider two points $x, y \in \Xs$. To apply the argument of \citet{Mokhtari2017conditional}, we would like a bound on the change in $\nabla h$ along the MFW update from $x$ in the direction of $y$. Let $x' = x + \frac{1}{T} y$ be the updated point. We have
 
 \begin{align*}
 \lVert \nabla h(\vec{F}(x)) - \nabla h(\vec{F}(x')) \rVert &= \left\lVert \left(\nabla h(\vec{F}(x))\right) D \vec{F}(x) - \left(\nabla h(\vec{F}(x'))\right) D \vec{F}(x') \right\rVert\\
  &= \left\lVert \left(\nabla h(\vec{F}(x))\right) D \vec{F}(x) - \left(\nabla h(\vec{F}(x))\right) D \vec{F}(x') \right. \\
  &\quad\quad + \left. \left(\nabla h(\vec{F}(x))\right) D \vec{F}(x') - \left(\nabla h(\vec{F}(x'))\right) D \vec{F}(x') \right\rVert \\
 &\leq \left\lVert \left(\nabla h(\vec{F}(x))\right) D \vec{F}(x) - \left(\nabla h(\vec{F}(x))\right) D \vec{F}(x') \right\rVert \\ 
 &\quad\quad + \left\lVert \left(\nabla h(\vec{F}(x))\right) D \vec{F}(x') - \left(\nabla h(\vec{F}(x'))\right) D \vec{F}(x') \right\rVert \\
 &= \left\lVert \left(\nabla h(\vec{F}(x))\right) \left( D \vec{F}(x) - D \vec{F}(x') \right) \right\rVert \\
 &\quad\quad + \left\lVert \left(\nabla h(\vec{F}(x)) - \nabla h(\vec{F}(x'))\right) D \vec{F}(x') \right\rVert.
 \end{align*}
 
 Starting out with the first term, we note that $\nabla h(\vec{F}(x))$ is a probability vector (the optimal $p$ for the DRO problem). Hence, we have 
 \begin{align*}
 \left\lVert \left(\nabla h(\vec{F}(x))\right) \left( D \vec{F}(x) - D \vec{F}(x') \right) \right\rVert &\leq \max_{i = 1...n} \left\lVert D \vec{F}(x)_i - D \vec{F}(x')_i \right\rVert\\
 &= \max_{i = 1...n} \left\lVert \nabla F_i(x) - \nabla F_i (x') \right\rVert
 \end{align*}
 
 And from Lemma 4 of \citet{Mokhtari2017conditional}, we have that when $x'$ is an updated point of the MFW algorithm starting at $x$, 
 \begin{align*}
 \left\lVert \nabla F_i(x) - \nabla F_i (x') \right\rVert \leq b \sqrt{k} \lVert x - x'\rVert \quad \forall i = 1...n.
 \end{align*}
 
 We now turn to the second term. Note that the $j$th component of this vector is just the dot product
 \begin{align*}
 \left(\nabla h(\vec{F}(x)) - \nabla h(\vec{F}(x'))\right) \cdot D \vec{F}(x)_{\cdot, j}
 \end{align*}
 
 where $D \vec{F}(x)_{\cdot, j}$ collects the partial derivative of each $F_i$ with respect to $x_j$. Via the Cauchy-Schwartz inequality, we have 
  \begin{align*}
 \left(\nabla h(\vec{F}(x)) - \nabla h(\vec{F}(x'))\right) \cdot D \vec{F}(x)_{\cdot, j} \leq \left\lVert \left(\nabla h(\vec{F}(x)) - \nabla h(\vec{F}(x'))\right) \right\rVert \left\lVert D \vec{F}(x)_{\cdot, j}\right\rVert
 \end{align*}
 
 Lemma \ref{lem:smooth-and-lipschitz-variance} shows that $\left\lVert \left(\nabla h(\vec{F}(x)) - \nabla h(\vec{F}(x'))\right) \right\rVert \leq L \lVert x - x' \rVert$. In order to bound the second norm, we claim that for all $i, j$, $\nabla_j F_i(x) \leq b$. To show this, note that we can use the definition of the multilinear extension to write
 \begin{align*}
 \nabla_j F_i(x) = \E_{S \sim x}[f_i(S|\{j\} \in S)] - \E_{S \sim x}[f_i(S|\{j\} \not\in S)]
 \end{align*}
 
 where $S \sim x$ denotes that $S$ is drawn from the product distribution with marginals $x$. Now it is simple to show using submodularity of $f_i$ that 
  \begin{align*}
 \E_{S \sim x}[f_i(S|\{j\} \in S)] - \E_{S \sim x}[f_i(S|\{j\} \not\in S)] \leq f_i(\{j\}) - f_i(\emptyset) \leq b.
 \end{align*}
 
 Accordingly, we have that 
\begin{align*}
\left\lVert D \vec{F}(x)_{\cdot, j}\right\rVert \leq b \lVert \mathbf 1\rVert = b \sqrt{n}.
\end{align*}

This gives us a component-wise bound on each element of the vector $\left(\nabla h(\vec{F}(x)) - \nabla h(\vec{F}(x'))\right) D \vec{F}(x')$. Putting it all together, we have
\begin{align*}
\left\lVert \left(\nabla h(\vec{F}(x)) - \nabla h(\vec{F}(x'))\right) D \vec{F}(x') \right\rVert &\leq b \sqrt{n} L \lVert x - x'\rVert \cdot \lVert \mathbf 1 \rVert\\
&\leq b \sqrt{n |V|} \cdot L \cdot \lVert x - x'\rVert,
\end{align*}

and summing the two terms yields the final Lipschitz constant $b \sqrt{n |V|} L + b \sqrt{k}$.
\end{proof}
Now the final convergence rate for MFW stated in Theorem \ref{lemma:mfw-dro} follows from plugging the above Lipschitz bound into Lemma \ref{lemma:mfw-convergence-exact}. We also remark that the above argument trivially goes through for an arbitrary (not necessarily submodular) functions:

\begin{lemma} 
	Suppose that each function $f: \reals^{|V|} \to \reals$ in the support of $P$ has bounded norm gradients $\max_{i = 1...|V|}|\nabla_i f| \leq b$ which are also $L_f$-Lipschitz. Then under the high variance condition, the corresponding DRO objective $G$ has $L_G$-Lipschitz gradient with $L_G \leq L_f + b\sqrt{n |V|} L$, where $L$ is as defined in Lemma \ref{lem:smooth-and-lipschitz-variance}.
\end{lemma}

\section{Rounding to a distribution over subsets}

The output of MFW is a fractional vector $x \in \Xs$. Lemma \ref{lemma:round} guarantees this $x$ can be converted into a distribution $\Ds$ over feasible subsets, and moreover, that the attainable solution value from doing so is within a $(1 - 1/e)$ factor of the optimal value for the DRO problem. This result is essentially standard (see \cite{wilder2018equilibrium} for a more detailed presentation), but we sketch the process here for completeness. There are two steps. First, we argue that $x$ can be converted into a distribution over subsets with equivalent value for the DRO problem. Second, we argue that the \emph{optimal} $x$ (product distribution) has value within $(1 - 1/e)$ of the optimal arbitrary distribution over subsets.

For the first step, our starting point is the swap rounding algorithm of \cite{chekuri2010dependent}. Swap rounding is a randomized rounding algorithm which takes a vector $x$ and returns a feasible subset $S$. For any single submodular function and its multilinear extension $F$, swap rounding guarantees $\E[f(S)] \geq F(x)$. In our setting, such guarantees cannot be obtained for a single $S$ since we want to simultaneously match the value of $x$ with respect to $n$ submodular functions $f_1...f_n$. However, swap rounding obeys a desirable concentration property which allows us to form a distribution $\Ds$ by running swap rounding independently several times and returning the empirical distribution over the outputs. Provided that we take sufficiently many samples, $\Ds$ is guaranteed to satisfy $\E_{S \sim \Ds}[f_i(S)] \geq F_i(x) - \epsilon$ for all $i=1...n$ with high probability. Specifically, \citet{wilder2018equilibrium} show that it suffices to draw $O\left(\frac{\log \frac{n}{\delta}}{\epsilon^3}\right)$ sets via swap rounding in order for this guarantee to hold with probability $1 - \delta$. 

The other piece of Lemma \ref{lemma:round} relates the optimal value for Problem \eqref{problem:continuous} (optimizing over product distributions) to the optimal value for the complete DRO problem (optimizing over arbitrary distributions). These values are easily shown to be within $(1 - 1/e)$ of each other by applying the correlation gap result of \cite{agrawal2010correlation}. For any product distribution $p$ over subsets, let $\text{marg}(p)$ denote the set of (potentially correlated) distributions with the same marginals as $p$. This result shows that for any submodular function $f$, 

\begin{align*}
\max_{p: \text{ a product distribution}} \max_{q \in \text{marg}(p)} \frac{\E_{S \sim q}[f(S)]}{\E_{S \sim p}[f(S)]} \leq \frac{e}{e - 1}
\end{align*}

and now Lemma \ref{lemma:round} follows by applying the correlation gap bound to each of the $f_i$.

\end{document}